\documentclass{article}

\usepackage{enumitem} 
\usepackage{amsmath, amsfonts, amsthm, amssymb}  
\usepackage{thmtools}
\usepackage{environ}  
\usepackage{natbib}
\usepackage{bbm}
\usepackage{mdframed}
\usepackage{multirow}

\usepackage{booktabs}
\usepackage{cite}
\usepackage{enumitem} 
\usepackage{multicol} 

\usepackage[table]{xcolor}
\usepackage{graphicx}
\usepackage{subcaption} 

\usepackage{float}     
\usepackage[table]{xcolor}

\def\E{\mathbb{E}}
\def\1{\mathbf{1}}

\def\Pr{\mathbb{P}}

\declaretheorem[name=Lemma]{lemma}
\declaretheorem[name=Definition]{definition}
\declaretheorem[name=Axiom]{axiom}

\usepackage{lipsum} 

\DeclareMathOperator*{\argmax}{arg\,max}
\DeclareMathOperator*{\argmin}{arg\,min}

\def\actions{\mathcal{A}}

\def\states{\mathcal{S}}

\def\trajectories{\mathcal{H}}
\def\traj{h}

\NewEnviron{answer}{

    -------- \\
    \BODY
}
\usepackage[a4paper, margin=1in]{geometry} 
\usepackage{parskip} 

\begin{document}

\title{Choice Between Partial Trajectories: \\Disentangling Goals from Beliefs}
\author{
    Henrik Marklund\thanks{Computer Science, Stanford University (\texttt{marklund@stanford.edu}).} \and 
    Benjamin Van Roy\thanks{Electrical Engineering and Management Science \& Engineering, Stanford University (\texttt{bvr@stanford.edu}).}
}
\date{}
\maketitle

\begin{abstract}
As AI agents generate increasingly sophisticated behaviors, manually encoding human preferences to guide these agents becomes ever more challenging. To address this, it has been suggested that agents instead learn preferences from human choice data.  This approach requires a model of choice behavior that the agent can use to interpret the data. For choices between partial trajectories of states and actions, previous models assume choice probabilities to be determined by the partial return or the cumulative advantage.

We consider an alternative model based instead on the bootstrapped return, which adds to the partial return an estimate of the future return.  Benefits of bootstrapped return over partial return and cumulative advantage models stem from their treatment of human beliefs.  Unlike partial return, choices based on cumulative advantage or bootstrapped return reflect human beliefs about the environment.  Further, while recovering the reward function from choices based on cumulative advantage requires that those beliefs are correct, doing so from choices based on bootstrapped return does not.  In this sense, the bootstrapped return model disentangles the human's goals from their beliefs.

To motivate the bootstrapped return model, we formulate axioms and prove an Alignment Theorem.  This result formalizes how, for a general class of human preferences, such models are able to disentangle goals from beliefs.  This ensures recovery of an {\it aligned} reward function when learning from choices based on bootstrapped return. 

The bootstrapped return model also affords greater robustness to choice behavior.  Even if human choices are based on partial return, learning via a bootstrapped return model recovers an aligned reward function.  The same holds with choices based on the cumulative advantage if the human and the agent both adhere to correct and consistent beliefs about the environment.  On the other hand, if choices are based on bootstrapped return, learning via partial return or cumulative advantage models does not generally produce an aligned reward function.
\end{abstract}

\section{Introduction}
\label{se:introduction}

To align an AI agent with human goals, it is common to design the agent to accumulate rewards that express human preferences.  Manual specification of an effective reward function, however, is notoriously difficult \citep{akrour2012april,hadfield2016cooperative}.  And mispecification leads to 
reward hacking.  Alignment can be improved via a reward function learned from human choice data \citep{akrour2012april,sadigh2017active, christiano2017deep,ibarz2018reward,brown2019deep}.

To generate choice data, it is common to present a human with pairs of action-observation trajectories and ask which they prefer.  Inferring the human's goals from this data requires a model of human choice behavior.  In existing models, the probability of choosing one trajectory over another is determined by scores assigned to the two, with each score calculated using a reward function that expresses the human's goals.   
\citet{christiano2017deep} popularized the use of partial return, expressed in the first row of Table \ref{tab:comparison_models}.  We use the term {\it partial} to indicate that this is the return over a partial trajectory, which is of finite length, as opposed to an infinite trajectory.  

\begin{figure}[H]
\centering
\includegraphics[width=4in]{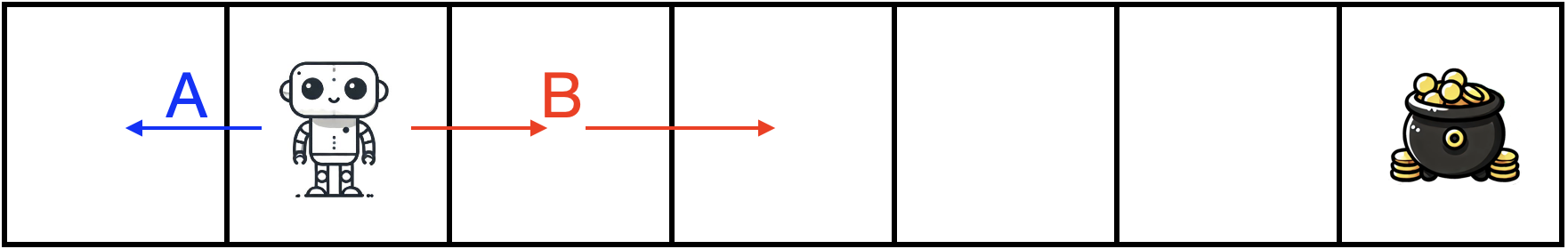}
\caption{The human wants the agent to reach the treasure quickly.  Trajectory A is preferred by a human who believes the grid wraps around so that after exiting to the left the agent appears on the right.  Trajectory B is preferred by a human who believes that moving left leads to a dead end.}
\label{fig:1dGrid}
\vspace{-.0cm}
\end{figure}

\begin{table}[h]
\centering
\begin{tabular}{|c|c|c|}
\hline
\textbf{score} & \textbf{formula} & \textbf{comment} \\ 
\hline
& & \\
partial return & $\sum_{t=0}^{T-1} r(s_t, a_t)$ & 
assumes choices uninfluenced by human beliefs \\
& & \\
\hline
& & \\
cumulative advantage & $\sum_{t=0}^{T-1} \Delta_*(s_{t},a_t)$ & 
infers goals that depend on human beliefs \\
& & \\
\hline
& & \\
bootstrapped return & $\sum_{t=0}^{T-1} r(s_t,a_t) + V(s_T)$ & 
addresses above concerns \\
& & \\
\hline
\end{tabular}
\caption{Trajectory scores used in choice models.}
\label{tab:comparison_models}
\end{table}

Choices based on partial return are not influenced by human beliefs about the environment.  This counters common sense.  For example, consider the grid environment of Figure \ref{fig:1dGrid}. Suppose the human's goal is for the robot to reach the treasure as quickly as possible. We can model this goal as the human attributing large reward to reaching the treasure and a small cost per timestep en route. Then, trajectory B incurs twice the cost of trajectory A; the partial returns are $-2$ for B and $-1$ for A.  In spite of that, it is natural for a human to prefer whichever trajectory maximizes progress toward the treasure.  Under reasonable beliefs about the environment, this could be trajectory B.  However, if the human believes that the grid wraps around -- so that after exiting to the left the agent appears on the right -- then A  is preferred.  This is because the human's quality assessment depends on beliefs about how each trajectory will subsequently play out.  As another example, consider a human choosing between chess trajectories. Assume that the human's sole goal is winning. It seems reasonable to suppose that the human's evaluation of a trajectory will be influenced by how the human thinks the game will subsequently play out, given the chess board at the end of the trajectory. More generally, it seems plausible that human choices are influenced by what they believe will be the continuation of a partial trajectory. This is in line with empirical results of \citep[Section 3.1.1]{christiano2017deep}, which demonstrate that when using the partial return model the learned reward functions express anticipated rewards. This suggests that, in practice, choices are made not only based on partial return, but also on anticipated future rewards.

To better reflect human choice behavior, \citet{knox2024models} proposed the cumulative advantage\footnote{\citet{knox2024models} actually proposed regret, which is the negation of cumulative advantage.  We use the latter so that we can more conveniently compare to other positive score formulas.}, as expressed in the second row of Table \ref{tab:comparison_models}.  Each advantage $\Delta_*(s_t,a_t) = Q_*(s_t,a_t) - V_*(s_t)$ measures the 
the difference between the optimal value $Q_*(s_t,a_t)$ of the executed action and the optimal state value $V_*(s_t)$.

When learning via the cumulative advantage model, a reward function is deduced from what is learned about advantages.  A weakness of this approach is in its reliance on correctness of the human's beliefs about the environment and an optimal policy.  In many applications, such an assumption is not tenable.  Indeed, one of the motivations for building AI systems is that they will assist in making effective decisions where humans are unable to. 

To elucidate on how cumulative advantage relies on correct beliefs, let us revisit Figure \ref{fig:1dGrid}.  Suppose the agent is currently positioned as in the figure.  Whether trajectory A or B is preferred depends on whether the human believes that the grid wraps around.  Let us assume that the grid does not wrap around, but that the human believes it does and therefore prefers trajectory $A$. Then, an agent with correct knowledge of the environment will infer from advantages a large reward to the left. Thus, the agent will infer a reward function that does not express the human's goals.

Motivated by these limitations, we consider an alternative model based on the bootstrapped return, as expressed in the third row of Table \ref{tab:comparison_models}.  The bootstrapped return adds to the partial return an estimate of expected future reward, expressed by a value function.  In the reinforcement learning literature, what we refer to as {\it the bootstrapped return} is often called the $N$-{\it step return} \citep{sutton2020reinforcement}.

The bootstrapped return with an optimal value function was briefly explored in the Appendix of \citet{knox2024models}.  We consider the concept more generally, with an arbitrary value function that expresses the human's beliefs about expected future reward.  Importantly, we do not assume that the human's value estimates are accurate.  Since the bootstrapped return depends on these estimates, choices reflect human beliefs.  Furthermore, as we will establish later in the paper, recovering the reward function via a bootstrapped return model does not require correctness of human beliefs.

Let us revisit the example of Figure \ref{fig:1dGrid} to illustrate these benefits over partial return and cumulative advantage.  While the partial return would encourage choice of A over B, either choice could be preferred according to the bootstrapped return.  This is because depending on the human's beliefs, the value estimate could be either large or small at the left end of the grid. As another example, consider a human choosing between chess trajectories. Under the bootstrapped return model, the human's evaluation of a trajectory will be influenced by how the human thinks the game will subsequently play out, given the chess board at the end of the trajectory.  In this way, the bootstrapped return accounts for the influence of human beliefs on choice.

Recall that when learning via the cumulative advantage model the agent can impute erroneous large reward at the left end of the grid.  This happens if the human believes the grid wraps around but the agent knows it does not and assumes the human knows that as well.  Suppose now that these beliefs are expressed through a bootstrapped reward model.  Will the agent infer the correct reward function unlike with cumulative advantage? The answer is yes. In this case, there will be a large value estimate at the left end to reflect proximity to the treasure.  And fitting a bootstrapped model to corresponding choice data recovers this value function and the human's reward function $r$.  In this sense, the bootstrapped return model disentangles goals, as expressed by $r$, from beliefs, which are expressed through the value function. This is crucial. For instance, in the game of chess, while the human may often choose trajectories in which the opponent's queen is captured, the agent ought to learn a reward function that reflects that the human's sole goal is to win. The human only wishes that the queen is captured in so far as the human believes that it will lead to winning down the line.

To motivate the bootstrapped return model, in the next section, we formulate axioms and prove an Alignment Theorem.  This result formalizes how, for a very general class of human preferences, such models are able to disentangles goals from beliefs.  In Section 3, we describe a practical algorithm for learning from stochastic choice data via a logit choice model based on bootstrapped return.  Computational results demonstrate efficient recovery of the reward function even when the human makes choices based on erroneous beliefs about the environment.  Section 4 discusses how properties of the bootstrapped return result in greater robustness to choice behavior when learning a reward function.  We discuss extensions and limitations in a closing section.

\section{Axiomatic Motivation}

We will establish that, under two axioms, a human's preferences between infinite trajectory lotteries are determined by their preferences between partial trajectories.  We begin by formally defining these terms and then state axioms and results.

\subsection{Trajectories and Lotteries}

Let $\mathcal{S}$ and $\mathcal{A}$ be finite state and action spaces.  We refer to any sequence of state-action pairs as a \textit{trajectory}.  An {\it infinite trajectory} takes the form $(s_0,a_0,s_1,a_1,\ldots)$.  Let $\trajectories_\infty$ be the set of all infinite trajectories.  A {\it partial trajectory} takes the form $(s_0,a_0,\ldots, a_{T-1},s_T)$, beginning and ending at states and lasting over any finite duration $T$. Let $\trajectories$ be the set of all partial trajectories.

A policy $\pi$ assigns a probability $\pi(a|s)$ to each action given a current state $s$.  A transition matrix $P \in \mathbb{R}^{|\mathcal{A} | \times |\mathcal{S}| \times |\mathcal{S}|}$ assigns a probability $P_{ass'}$ over next state $s'$ given a current state-action pair $(s,a)$.

Each state $s$, policy $\pi$, and transition matrix $P$, induce a distribution $\mathbb{P}_{\pi, P}(\cdot | s)$ over infinite trajectories.  For any partial trajectory $h = (s_0,a_0,\ldots,s_T)$, let $\mathbb{P}_{\pi, P}(\cdot | h)$ be the distribution over infinite trajectories that begin with $h$.  In particular, if $H_\infty = (S_0,A_0,S_1,A_1,\ldots)$ is sampled from $\mathbb{P}_{\pi, P}(\cdot | h)$ then $(S_0,A_0,\ldots,S_T) = h$, and $(S_T, A_T, S_{T+1}, \ldots)$ is distributed according to $\mathbb{P}_{\pi, P}(\cdot | S_T)$.

A {\it lottery} over infinite trajectories is a probability distribution over infinite trajectories.  For example, for any policy $\pi$, transition matrix $P$, and partial trajectory $h$, the probability distribution $\mathbb{P}_{\pi, P}(\cdot | h)$ is a lottery.  Given a lottery $\ell$ and function $f:\trajectories_\infty \rightarrow\Re$, we denote by $\E_\ell[f(H_\infty)]$ the expectation of $f(H_\infty)$ with $H_\infty = (S_0,A_0,S_1,A_1,\ldots)$ sampled from $\ell$.  Similarly, we denote by $\E_{\pi,P}[f(H_\infty)|h]$ the expectation with respect to the lottery $\mathbb{P}_{\pi, P}(\cdot | h)$.

To simplify analysis we will assume that a notion of {\it relative state frequency} is defined for each lottery.  In particular, we will restrict attention to lotteries $\ell$ for which the limit $\lim_{T \to \infty }\E_\ell[\frac{1}{T} \sum_{t=0}^{T-1} \1(S_t = s)]$ exists for each $s \in \mathcal{S}$.  Let $\mathcal{L}_\infty$ be the set of such lotteries.

\subsection{Preferences between Lotteries}

Let $\succeq$ be a binary relation on $\mathcal{L}_\infty$. We interpret this relation as indicating the human's preferences between infinite trajectory lotteries. For each infinite trajectory lottery $\ell \in \mathcal{L}_\infty$, discount factor $\gamma \in (0,1)$, and function $r:\states\times\actions \rightarrow \Re$, let
$$
v_{\ell, r, \gamma} = \mathbb{E}_{\ell}\left[\sum_{t=0}^\infty \gamma^t r(S_{t}, A_{t})\right].
$$
We interpret this as the expected discounted return of the lottery.  A reward function expresses preferences between infinite trajectory lotteries in the following sense.
\begin{mdframed}
\begin{definition}
\label{def:express-infinite}
{\bf (expressing $\succeq$)}
A function $r:\states\times\actions \rightarrow \Re$ is said to express $\succeq$ if, for all $\ell, \ell' \in \mathcal{L}_\infty$,
$$\ell \succeq \ell' \qquad \text{if and only if} \qquad \lim_{\gamma \uparrow 1} (v_{\ell, r, \gamma} - v_{\ell', r, \gamma}) \geq 0.$$
\end{definition}
\end{mdframed}
Loosely speaking, when $r$ expresses $\succeq$, lotteries with higher expected cumulative reward are preferred. The above preference structure corresponds to comparing lotteries based on a so-called bias-optimality criterion \citet[See Section 5.4.3]{puterman2014markov}. If one lottery has a higher average reward, that lottery is preferred. If the two lotteries have equal average reward, then the lottery with the highest so-called bias is preferred.

\subsection{Preferences between Partial Trajectories}

In practice, we cannot elicit preferences between infinite trajectory lotteries.  This is because infinite trajectories -- let alone distributions over infinite trajectories -- comprise too much data for a human to process.  Instead, we can elicit preferences between partial trajectories.  This can be done, for example, by presenting pairs of partial trajectories to a human and observing choices.

To model what drives these choices, we consider a binary relation $\succeq_\partial$ between partial trajectories $\trajectories$.  We interpret $\succeq_\partial$ as indicating the human's preferences between partial trajectories.

We now define what it means for a pair of functions $(r,V)$ to express preferences between partial trajectories.
\begin{mdframed}
\begin{definition}
\label{def:express-partial}
{\bf (expressing $\succeq_\partial$)}
A pair $(r,V)$, comprised of functions $r:\states\times\actions \rightarrow \Re$ and $V:\states\rightarrow\Re$, is said to express $\succeq_\partial$ if, for all $h,h'\in \mathcal{H}$,
$$h \succeq_\partial h' \qquad \text{if and only if} \qquad \sum_{t=0}^{T-1} r(s_t,a_t) + V(s_T) \geq \sum_{t=0}^{T'-1} r(s'_t,a'_t) + V(s'_{T'}),$$
where $(s_0,a_0,\ldots,s_T) = h$ and $(s'_0,a'_0,\ldots,s'_{T'}) = h'$.
\end{definition}
\end{mdframed}
We refer to the quantity $\sum_{t=0}^T r(s_t,a_t) + V(s_T)$ as a {\it boostrapped return}.  It is bootstrapped in the sense that it relies on a value $V(s_T)$ that can be interpreted as approximating subsequent return.  In the definition, $(r,V)$ expresses preferences via comparing bootstrapped returns.

\subsection{Alignment}
In this section, we relate preferences over partial trajectories to preferences over infinite trajectory lotteries.  We do so through proving a result that relies on two axioms: one constrains $\succeq$ and the other constraints how $\succeq_\partial$ relates to $\succeq$.

Our first axiom requires that $\succeq$ can be expressed by a reward function.  One could alternatively assert more primitive axioms, along the lines considered in rational choice theory \citep{von1947theory,koopmans1960stationary,koopmans1972representation,bastianello2019time}, that imply reward representation.  We forgo that and simply assert reward representation itself as an axiom in order to focus attention on what is required beyond that to form our theory of alignment.
\begin{mdframed}
\begin{axiom}
\label{ax:reward-representation}
{\bf (reward representation)}
There exists a function $r:\states\times\actions \rightarrow \Re$ that expresses $\succeq$.
\end{axiom}
\end{mdframed}

Our next axiom requires a form of alignment -- or consistency -- between preferences over partial trajectories and preferences over infinite trajectory lotteries.

This axiom refers to a pair $(\pi, P)$ that for each trajectory $h$ induces a lottery $\Pr_{\pi, P}(\cdot |h) \in \mathcal{L}_\infty$. This lottery generates an infinite trajectory that begins with $(S_0,A_0,\ldots,S_T) = h \in \trajectories$ and then follows a Markov process with state space $\states\times\actions$.

The axiom refers to a set $\mathcal{P}$, which we define to be the set of pairs $(\pi,P)$ for which $\lim_{T \rightarrow \infty} \E_{\pi,P}[\frac{1}{T} \sum_{t=0}^{T-1} \1(S_t=s') | s]$ does not depend on $s \in \states$.  In other words, $(\pi,P) \in \mathcal{P}$ if the initial state of the Markov chain does not impact expected relative state frequencies.  In other words, the Markov chain has a unique stationary distribution.
\\
\begin{mdframed}
\begin{axiom}
\label{ax:alignment}
{\bf (alignment)}
There exists $(\pi, P) \in \mathcal{P}$ such that, for all partial trajectories $h$ and $h'$,
\begin{align*}
h \succeq_\partial h' && \mathrm{if\ and\ only\ if} && \mathbb{P}_{\pi, P}(\cdot | h) \succeq \mathbb{P}_{\pi, P}(\cdot | h').
\end{align*}
\end{axiom}
\end{mdframed}

To interpret this axiom, think of $(\pi, P)$ as how the human imagines the future to be generated.  The axiom says that the human will prefer a partial trajectory $h$ over $h'$ if and only if the infinite trajectory lottery $\mathbb{P}_{\pi, P}(\cdot | h)$ is preferred over $\mathbb{P}_{\pi, P}(\cdot | h')$.  The requirement that $(\pi, P) \in \mathcal{P}$ can be interpreted as assuming the human believes that each partial trajectory bears transient impact on the future state evolution.

The following result, which relates $\succeq$ to $\succeq_\partial$, provides our axiomatic motivation for modeling preferences in terms of bootstrapped return.
\begin{mdframed}
\begin{restatable}[]{theorem}{alignment}
\label{th:alignment}
{\bf (alignment)}
Suppose that $(\succeq, \succeq_\partial)$ satisfies Axioms \ref{ax:reward-representation} and \ref{ax:alignment}.\\
\begin{tabular}{ll}
1. & There exists a pair $(r, V)$ that expresses $\succeq_\partial$. \\
2. & If $(r,V)$ expresses $\succeq_\partial$ then $r$ expresses $\succeq$.
\end{tabular}
\end{restatable}
\end{mdframed}
The first part of this theorem indicates that preferences between partial trajectories are determined by bootstrapped returns, calculated via some reward function and value function.  The second part indicates that this reward function, in turn, determines preferences between lotteries.

Theorem \ref{th:alignment} is proved in the appendix.  Within the proof, a particular expression $(\tilde{r}_{\pi,P,r}, \tilde{V}_{\pi,P,r})$ surfaces as a proof of existence that justifies Statement 1 of Theorem \ref{th:alignment}.  In particular, $(\tilde{r}_{\pi,P,r}, \tilde{V}_{\pi,P,r})$ expresses $\succeq_\partial$.  Given a reward function $r$ that expresses $\succeq$, the first element of this pair is a relative reward function $\tilde{r}_{\pi,P,r}(s,a) = r(s,a) - \overline{r}_{\pi,P,r}$, which measures reward relative to the average reward induced by $(\pi,P)$.  The second element is the relative value function $\tilde{V}_{\pi,P,r}$, which measures expected cumulative relative reward.  For each state $s$, $\tilde{V}_{\pi,P,r}(s)$ is finite because the expected reward approaches its average.

Another notable implication is that the relative reward function $\tilde{r}_{\pi,P,r}$ expresses $\succeq$.  Hence, preferences between infinite trajectory lotteries are determined by preferences between partial trajectories, encoded in terms of the relative reward function $\tilde{r}_{\pi,P,r}$ and the relative value function $\tilde{V}_{\pi,P,r}(s)$.  Note that the relative value function can be ignored, $\tilde{r}_{\pi,P,r}$ by itself expresses $\succeq$.

Preferences $\succeq$ between infinite trajectory lotteries do not necessarily depend on $\pi$ and $P$, which represent what the human imagines might transpire.  It may therefore seem remarkable that learning the relative reward function $\tilde{r}_{\pi,P,r}$, which does depend on $\pi$ and $P$, can reveal $\succeq$.  This is possible because if $\succeq$ is expressed by a reward function $r$ that does not depend on $\pi$ and $P$ then $\tilde{r}_{\pi,P,r}(s,a) = r(s,a) - \overline{r}_{\pi,P,r}$ must express the same preferences.  It is only the constant offset $\overline{r}_{\pi,P,r}$ that depends on $\pi$ and $P$. We note that the reason the offset is constant, rather than varying with state, is due to the assumption that $(\pi,P) \in \mathcal{P}$.

\subsection{Fictitious States and Actions}

A pair of preferences $(\succeq, \succeq_\partial)$ can violate Axiom \ref{ax:alignment} though they ought to satisfy the conclusions of Theorem \ref{ax:alignment}, which ensure that $\succeq_\partial$ can be expressed by some $(r,V)$ and that the reward function $r$ is aligned with $\succeq$.  At a high level, the issue is that our current axiom demands a kind of consistency between $V$ and $r$. For instance, if $V=0$ everywhere, that implies $r$ must be $0$ everywhere also. This is an unnecessarily strong restriction and could motivate weakening the axiom.  However, the issue can alternatively be addressed by expanding the state and action spaces.  We now explain the nature of such preferences and how to address them.

We begin with an example.  Consider two states $\states = \{1,2\}$ and a single action $\actions = \{1\}$.  Suppose $\succeq$ is expressed by a function $r$ for which $r(1,1) = 1$ and $r(2,1) = -1$.  Suppose $\succeq_\partial$ is expressed by $(r,V)$, with $V(s) = 0$ for $s \in \{1,2\}$.  Then, Axiom \ref{ax:alignment} is violated.  To see why, consider $h=(1,1,2)$ and $h'=(1)$.  The bootstrapped returns are $r(1,1) + V(2) = 1$ and $V(1) = 0$.  Hence, $h \succ_\partial h'$.  Consider lotteries $\ell$ and $\ell'$ induced by starting with either $h$ or $h'$ and then following $(\pi,P)$.  For all $(\pi,P) \in \mathcal{P}$, the lottery that begins with $h'$ ought to be at least as desirable as the one that begins with $h$.  This is because $h$ started as $h'$ and realized a worst-case transition, which was to state $2$.  As such, it ought to be the case that $\ell \preceq \ell'$, violating Axiom \ref{ax:alignment}.

A modification that incurs no consequential change to our example yields one that satisfies the axiom.  In particular, introduce a fictitious state $3$ and a fictitious action $2$ to produce enlarged state and action spaces $\states = \{1,2,3\}$ and $\actions = \{1,2\}$.  Let the reward and value functions be as before where already defined, and to assign values where they are not defined, let $r(\cdot,2) = r(3,\cdot) = 0$ and $V(3) = 0$.  Let $\succeq_\partial$ expressed by $(r,V)$.  It is easy to verify that Axiom \ref{ax:alignment} is satisfied with a policy $\pi$ that always selects action $2$ and a matrix $P$ that always transitions to state $3$.  Hence, Theorem \ref{th:alignment} holds, and thus $r$ expresses $\succeq$.  If we restrict attention to trajectories that avoid the fictitious state, $r$ still identifies preferences between lotteries over such trajectories.  Hence, while $r$ is defined over larger state and action spaces, restricting the domain gives us a reward function that is aligned with $\succeq$ in our original example.

Our example illustrated how enlarging a state space addresses situations where Axiom \ref{ax:alignment} is violated yet preferences ought to satisfy the conclusions of Theorem \ref{th:alignment}.  The following corollary accommodates more generally the sort of enlargement we considered in the example.  We use the term {\it on $(\states', \actions')$} to indicate restriction to trajectories that only include states and actions in $\states'$ and $\actions'$.
\begin{mdframed}
\begin{restatable}[]{corollary}{alignment_co}
\label{co:alignment}
{\bf (alignment on a subset)}
Suppose that $(\succeq, \succeq_\partial)$ satisfies Axioms \ref{ax:reward-representation} and \ref{ax:alignment}.  For any $\states' \subseteq \states$ and $\actions'\subseteq\actions$: \\
\begin{tabular}{ll}
1. & There exists a pair $(r, V)$ that expresses $\succeq_\partial$ on $(\states', \actions')$. \\
2. & If $(r,V)$ expresses $\succeq_\partial$ on $(\states', \actions')$ then $r$ expresses $\succeq$ on $(\states', \actions')$.
\end{tabular}
\end{restatable}
\end{mdframed}
This corollary follows almost immediately from Theorem \ref{th:alignment}.  To illustrate its use, let us revisit our example.  Let $\states'$ and $\actions'$ be the original state and action spaces, before enlargement to $\states$ and $\actions$.  Our original example began with relations $\succeq$ and $\succeq_\partial$ on $(\states',\actions')$, with the latter expressed by some $(r,V)$.  The corollary ensures that $r$ is aligned with $\succeq$ on $(\states',\actions')$, which is the same conclusion we drew from arguments specific to this example made before stating this general corollary.

This corollary justifies our treatment later in the paper of reward learning from choices determined by preferences that do not satisfy Axiom \ref{ax:alignment}.  For example, consider what happens when a human chooses whichever partial trajectory maximizes partial return with respect to some reward function $r$, which also expresses $\succeq$.  This is equivalent to bootstrapped return with a value function $V$ that assigns zero value to every state.  The preference relations violate Axiom \ref{ax:alignment}, but this can be addressed by expanding the state and action spaces and conjuring Corollary \ref{co:alignment} to ensure that a reward function learned from choices aligns with $\succeq$.

\section{Reward Learning}
\label{se:reward-learning}

In this section, we will discuss an approach to learning a reward function from choice data.  The point is for this reward function to express a human's preferences between infinite trajectory lotteries.  Such a reward function enables alignment of agent behavior with the human's preferences.  In particular, by increasing expected cumulative reward, the agent better serves the human's interests.

\subsection{Choice Data}

Our approach learns from choice data.  By this we mean a set $\mathcal{D}$ of triples $(\traj, \traj', y) \in \mathcal{D}$, each comprising two partial trajectories $\traj$ and $\traj'$ and a choice $y \in \{0,1\}$.  A choice of $\traj$ is indicated by $y=1$, while a choice of $\traj'$ is indicated by $y=0$.

We interpret this dataset as a set of queries, each consisting of a pair of partial trajectories, presented to a human, who then chooses between the two.  A reward learning algorithm takes this data as input and produces a reward function as output.

\subsection{Deterministic Choice}

In order to learn from choice data, we need to interpret how the human makes choices.  In the simplest, perhaps idealized, case, we could assume that each choice $y$ perfectly reflects the humans preference between partial trajectories.  In particular, $y = 1$ if $\traj \succ_\partial \traj'$ and $y = 0$ if $\traj' \succ_\partial \traj$.  Otherwise, if the human is indifferent, the choice $y$ is randomly sampled, taking value $0$ or $1$ with equal probability.

The set $\mathcal{H} \times \mathcal{H}$ of partial trajectory pairs is countably infinite.  Suppose that, as $\mathcal{D}$ is made up of random partial trajectory pairs and choices made in response, sampled such that each pair eventually appears infinitely often.  Then, as $\mathcal{D}$ grows, the data identify the preference relation $\succeq_\partial$.  By Theorem \ref{th:alignment}.1, there exists a pair $(r,V)$ that expresses $\succeq_\partial$, and thus, $\mathcal{D}$ identifies at least one such pair $(r,V)$.  By Theorem \ref{th:alignment}.2, $r$ then expresses $\succeq$.  In this way, from $\mathcal{D}$ it is possible to learn a reward function that enables alignment.

\subsection{Stochastic Choice}

In practice, it is unrealistic to expect a human to consistently make choices that perfectly match their preferences.  For instance, if the human is almost indifferent between two trajectories, their choices may vary if presented with the same two trajectories at different points in time.

Assuming choices to be stochastic is more realistic.  To model stochastic choice over partial trajectories, we use a choice probability function $p: \trajectories \times \trajectories \to [0,1]$.  We interpret $p(h,h')$ as the probability with which the human chooses $h$ over $h'$ when presented with these two partial trajectories.  Each choice is sampled independently according to this probability.  Note that $p(h,h') + p(h',h) = 1$. 

For choices to convey preferences, they need to reflect those preferences more often than not.  Perhaps the weakest assumption of this sort that enables learning preferences from choices is that, for all partial trajectories $h,h' \in \trajectories$, 
$$p(h,h') \geq 0.5 \qquad \mathrm{if\ and\ only\ if\ } \qquad h \succeq_\partial h'.$$
Deterministic choice is a special case in which $p(h,h') = 1$ if $h \succeq_\partial h'$.  Under this assumption, the preference relation $\succeq$ can again be learned from data generated by any deterministic or stochastic choice model.  In particular, suppose that, as $\mathcal{D}$ grows, choices between any pair $h$ and $h'$ can be averaged to identify $p(h,h')$ and thus the human's preference between $h$ and $h'$.  By identifying preferences between all pairs, we can again obtain a reward function that identifies $\succeq$ and thus enables alignment between agent and human.

One well versed in choice theory may wonder whether stronger assumptions on the stochastic relation between preferences and choice ought to be required to identify $\succeq$.  For example, it is common to require that choices not only between outcomes but also between lotteries be observed.  This is not needed in our context because the random mixing induced by lotteries can instead be induced by temporal mixing across an infinite trajectory.  That is, an infinite trajectory effectively serves as a lottery in which probabilities are determined by relative state frequencies.  While we do not consider requiring the human to compare infinite trajectories, long partial trajectories can approximate the effect arbitrarily well.

Stronger assumptions about the stochastic choice model can ensure identifiability even with choices between partial trajectories that are short.  For example, this is the case with the logit model considered in the next section.

\subsection{The Logit Choice Model}

To reduce the amount of data and computation required for reward learning, it is common to assume that $p$ resides in a particular function class.  We refer to such a function class as a {\it stochastic choice model}.  The most common is the logit choice model. In the logit choice model, each partial trajectory in $\mathcal{H}$ is assigned a numerical score.  In our case, we will take this score to be the bootstrapped return for some reward function $r$ and value function $V$. Then, the choice probability for partial trajectories $h$ and $h'$ is given by
$$p(h,h')  = \sigma \left(\sum_{t=0}^{T-1} r(s_t,a_t) + V(s_T) - \sum_{t=0}^{T'-1} r(s'_t,a'_t) - V(s'_{T'}) \right),$$
where $(s_0,a_0,\ldots,s_T) = h$, $(s'_0,a'_0,\ldots,s'_{T'}) = h'$, and $\sigma: \mathbb{R} \to (0,1)$ is the standard logistic function.  This is sometimes referred to as the Bradley-Terry model for binary choice with unit temperature.

Data and computation requirements can be reduced further by constraining the the functions $r$ and $V$.  In particular, these we can take these functions to be parameterized by a vector $\theta \in \Re^K$.  For example, $r_\theta(s,a)$ and $V_\theta(s)$ could be outputs of a neural network architecture with tunable parameters $\theta$ that takes $s$ and $a$ as inputs.  In this case, the choice probability model can itself be viewed as a neural network that takes partial trajectories $h$ and $h'$ as inputs and produces an output
$$\hat{p}_\theta(h,h')  = \sigma \left(\sum_{t=0}^{T-1} r_\theta(s_t,a_t) + V_\theta(s_T) - \sum_{t=0}^{T'-1} r_\theta(s'_t,a'_t) - V_\theta(s'_{T'}) \right).$$
Parameters can then be computed by minimizing the negative log-likelihood
\begin{align}
\label{eq:loss-function}
\ell(\theta| \mathcal{D}) = - \sum_{(h,h',y) \in \mathcal{D}} (y \log \hat{p}_\theta(h,h') + (1-y) \log \hat{p}_\theta(h',h)).
\end{align}
This loss can be approximately minimized, for example, via stochastic gradient descent.

A special case of this model uses a {\it tabular representation} of reward and value functions.  Each parameter encodes either a reward $r_\theta(s,a)$ assigned to a state-action pair $(s,a)$ or a value $V_\theta(s)$ assigned to a state $s$.  Hence, there are $|\states \times \actions| + |\states|$ parameters.  In other words, $\theta \in \Re^{|\states \times \actions| + |\states|}$.  While the resulting stochastic choice model $\hat{p}_\theta$ is practical only when the number of states is small enough so that $\theta$ can be stored and updated with reasonable memory and computation, it serves as a useful didactic example.

\subsection{Computational Example}

We will now study reward learning via a logit choice model based on bootstrapped return with a tabular representation of reward and value functions.  Our aim is to generate insights that extend beyond tabular representations.  In particular, our study demonstrates that:
\begin{enumerate}
\item A reward function that aligns with human preferences can be recovered from choice data even if the human makes choices based on erroneous beliefs about the environment.  In particular, we recover the same reward function from two humans who share goals but make different choices due to differing beliefs.
\item When using the logit choice model, practical learning algorithms infer the reward function, with accuracy diminishing quickly as choice data accumulates.  This is true even if partial trajectories in each pair are short.  And even if partial trajectories in each pair terminate at different states.  But learning can be accelerated by selecting each pair to share a common terminating state.
\end{enumerate}

We consider two humans who are trying to communicate their goals to a robot.  The two humans have a common goal, but have very different beliefs about how the robot can best achieve this goal. Because the two humans have different beliefs, they will make different choices when asked to compare partial trajectories.

\begin{figure}[htbp]
\centering
\begin{subfigure}[c]{.32\textwidth}
  \centering
  \vspace{-0.11in}
  \includegraphics[width=1.0\textwidth]{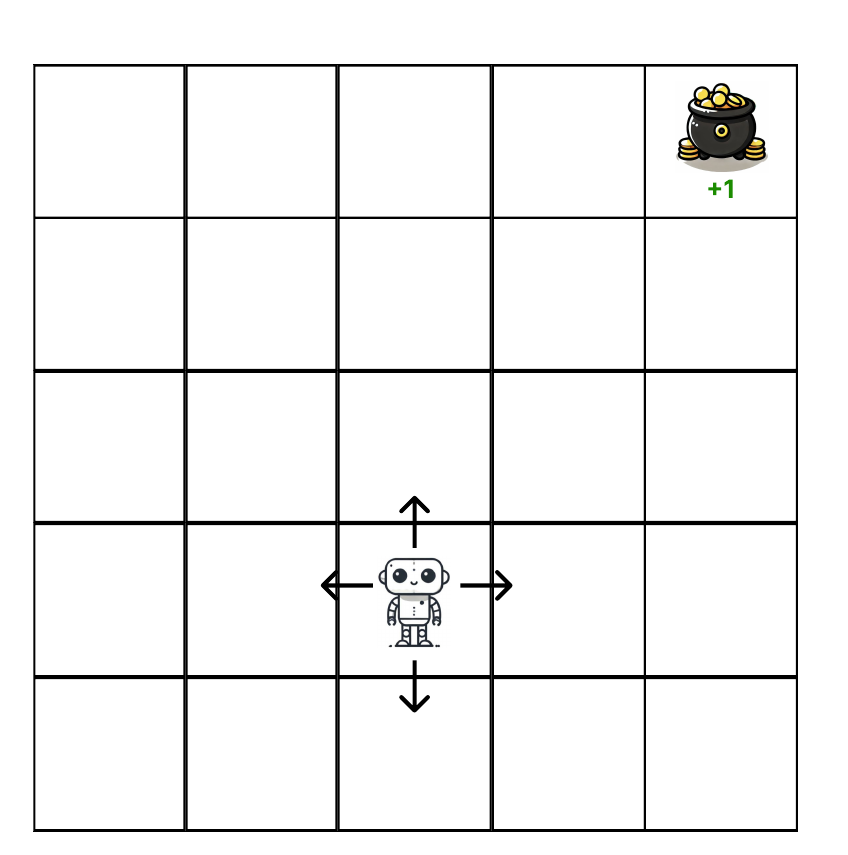}
  \caption{}
\end{subfigure}
\begin{subfigure}[c]{.32\textwidth}
  \centering
  \includegraphics[width=1.0\textwidth]{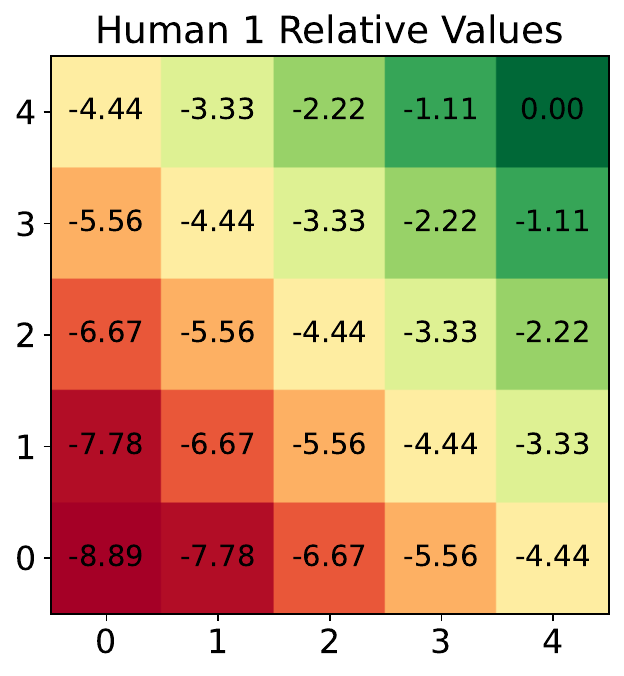}
  \caption{}
\end{subfigure}
\begin{subfigure}[c]{.32\textwidth}
  \centering
  \includegraphics[width=1.0\textwidth]{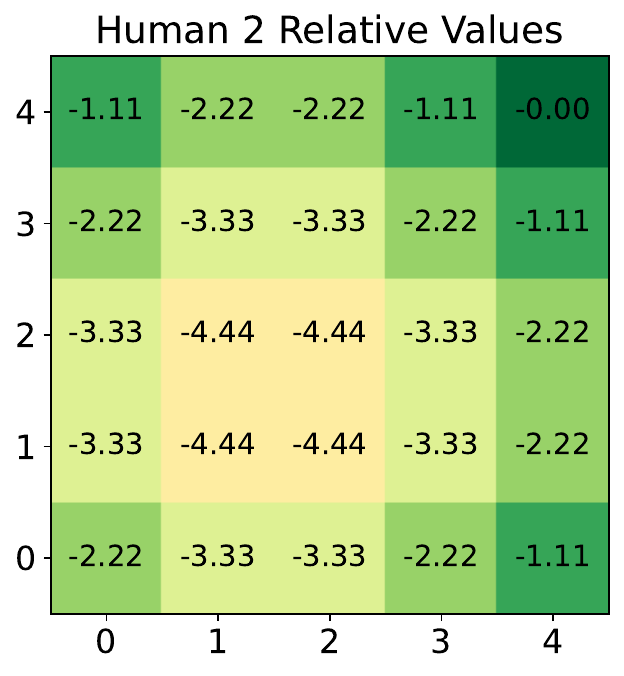}
  \caption{}
\end{subfigure}    
    \caption{\textbf{(a)} A grid environment in which the robot can move up, down, left and right, or stay put. The two humans have a common goal of the robot reaching the top right corner. \textbf{(b)} A visualization of the first human's relative value function. This human does not believe the robot can move through walls. \textbf{(c)} A visualization of the second human's relative value function. This human thinks the environment ``wraps around'' such that it's possible to move through a wall and come out the other end. That is why high value is assigned at corners.}
    \label{fig:grid-and-two-relative-value-functions}
\end{figure}

The robot acts on a two-dimensional grid with coordinates ranging from $(0,0)$ to $(4,4)$, as illustrated in Figure \ref{fig:grid-and-two-relative-value-functions}.  Both humans wish for the robot to go the top right corner, which we call the goal state, and then stay there. Specifically, each human's preference relation over lotteries is expressed by a reward function $r: \mathcal{S} \times \mathcal{A} \to \Re$, which assigns reward $1$ to the goal state and $0$ everywhere else. For simplicity, we assume that the human's reward function does not depend on the action, but only on state. This is equivalent to assuming that $|\mathcal{A}|=1$. 

\textbf{Differing beliefs.}
The two humans share the following beliefs about the dynamics of the environment. Both humans know that the robot can only move in one of the four directions, or stay put. The humans also agree that at each timestep the robot progresses towards the goal, in the sense of reducing the minimal number of steps to the goal, with $90 \%$ probability. With $10 \%$ probability, the agent stays put. If there are multiple next states that progress toward the goal, these are assigned equal probability. Once the robot reaches the goal, it stays there forever. 

There is one important difference between the two humans, however. While the first human does not believe the robot can move through walls, the second human does. In particular, the second human believes that the environment ``wraps around'' in the sense that the robot can move through a wall and emerge at the other end. As a result, the second human will tend to place a high value on corner states, since in the mind of this human, there is always a very short path from a corner to the goal state. This difference in belief results in two very different relative value functions. See a visualization in Figure \ref{fig:grid-and-two-relative-value-functions}b and \ref{fig:grid-and-two-relative-value-functions}c. This difference in relative value function will, in turn, result in different choices over partial trajectories. 

\textbf{Choice Data.} To communicate their goals to the robot, each human provides a choice for each pair of trajectories presented to them. Each trajectory in a pair is of length $3$ and is sampled in the following way: The starting state is sampled uniformly across the $5 \times 5$ states. The second state is sampled uniformly across the contiguous states: the state above, below, to the left, and to the right. If the robot is by a wall, and samples a state outside the grid, the robot stays put. Each subsequent state is sampled in the same manner. Thus, two consecutive states in a trajectory will always be contiguous in the grid.

Each human makes choices according to the logit choice model with the bootstrapped return as the score. First, the human scores the two partial trajectories using its reward function and relative value function. Then, the choice is sampled using the logit choice model.

\textbf{Estimating $r$}. 
We estimate $r$ and $V$ using SGD with the loss function in Equation \ref{eq:loss-function}. To evaluate how good the estimate of the reward function is, we scale and shift the learned reward functions such that $0$ is the lowest value, and $1$ is the highest. We then compute the root mean squared error (RMSE).

\textbf{Experiments.}
We perform a series of experiments examining how quickly the common reward function can be recovered, under different conditions.

First, we ask: How quickly can the reward function be recovered when trajectories are short? To study this question, we investigate how the RMSE varies as the number of data points increases. We do so for each of the two choice datasets generated by the two humans, respectively. For this experiment, to make the learning problem more challenging, we remove all trajectories that terminate in the same state. We plot the results in Figure \ref{fig:rmse_two_plots}a. We find that we need roughly $15,000$ samples to achieve RMSE of $0.1$. We also find that for both choice datasets, we recover the common reward function despite the two datasets being based on different beliefs of the two humans.

While the reward function can, in principle, be recovered regardless of the human's relative value function, it is plausible that the choice of value function affects the speed of learning. In this next study, we investigate this hypothesis. In particular, we ask: How does the magnitude of the relative value function affect the speed of learning? To study this question, we take the first human's relative value function and scale it by a factor $c > 1$. Based on this new relative value function, we regenerate the choice data and estimate the reward function. We redo this procedure for each scaling factor $c \in \{2,4,8\}$ and plot the results in Figure \ref{fig:rmse_two_plots}b. We find that learning is significantly slowed down as the magnitude of the relative value function is increased.

The negative impact of the value function on learning speed raises the question of whether this impact can be somehow reduced. The answer is yes. We can remove the impact of the value estimate by having the human only choose between trajectories that terminate at the same state. In this case, the value estimates cancel out and therefore do not influence the human's choice. This phenomenon has previously been recognized by \citet{knox2024models}. Because value estimates cancel out, it seems plausible that the reward function can be recovered more quickly by having the human provide choices between trajectories that terminate at the same state.  We investigate this hypothesis by fixing the end state across all trajectories. We plot the results in Figure $\ref{fig:rmse_two_plots}$b. We find that this speeds up learning. This benefit becomes especially pronounced as the scale of the relative value function increases.

\begin{figure}[htbp]
    \includegraphics[width=0.49\textwidth]{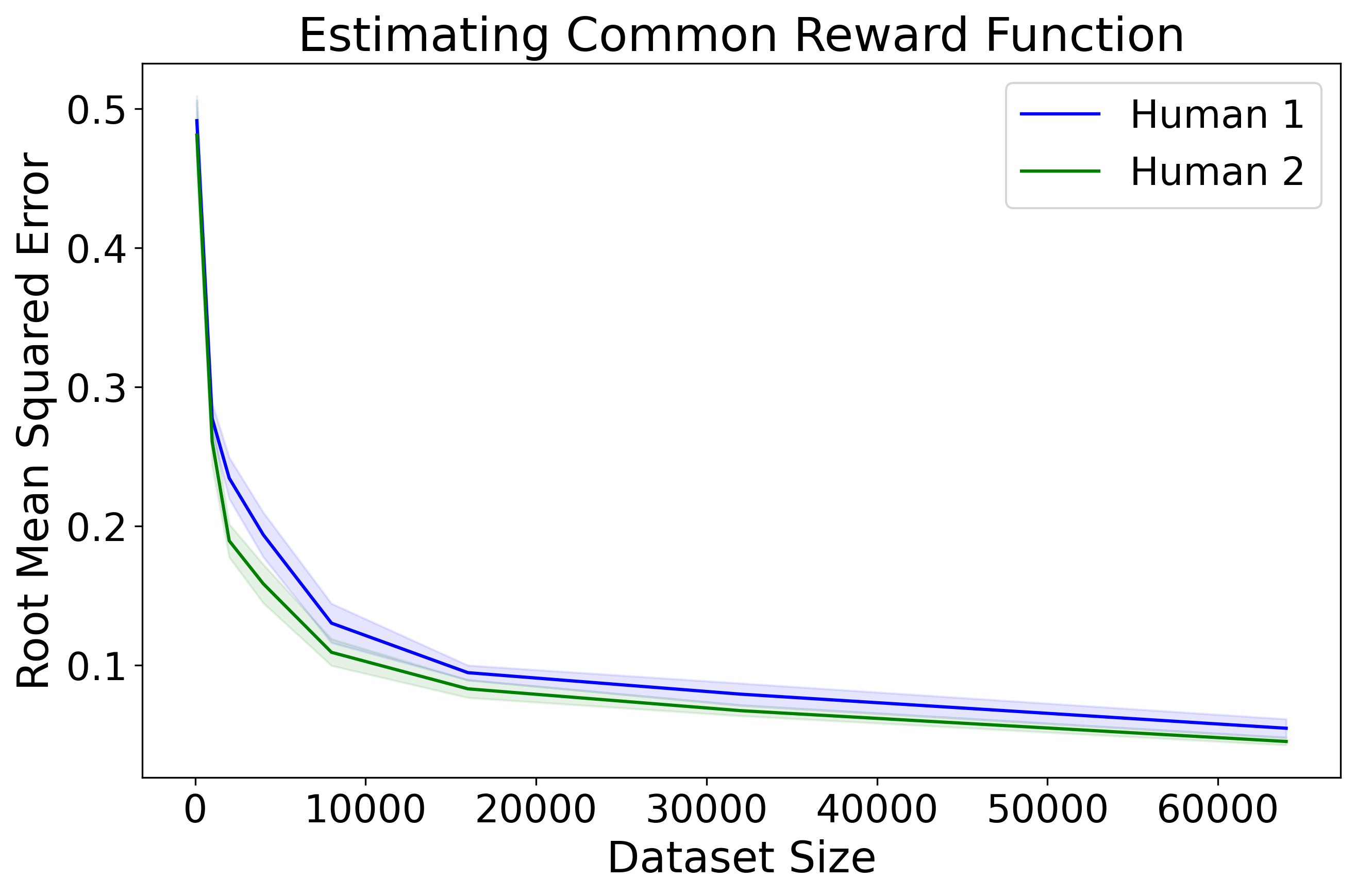}
    \hfill
    \includegraphics[width=0.49\textwidth]{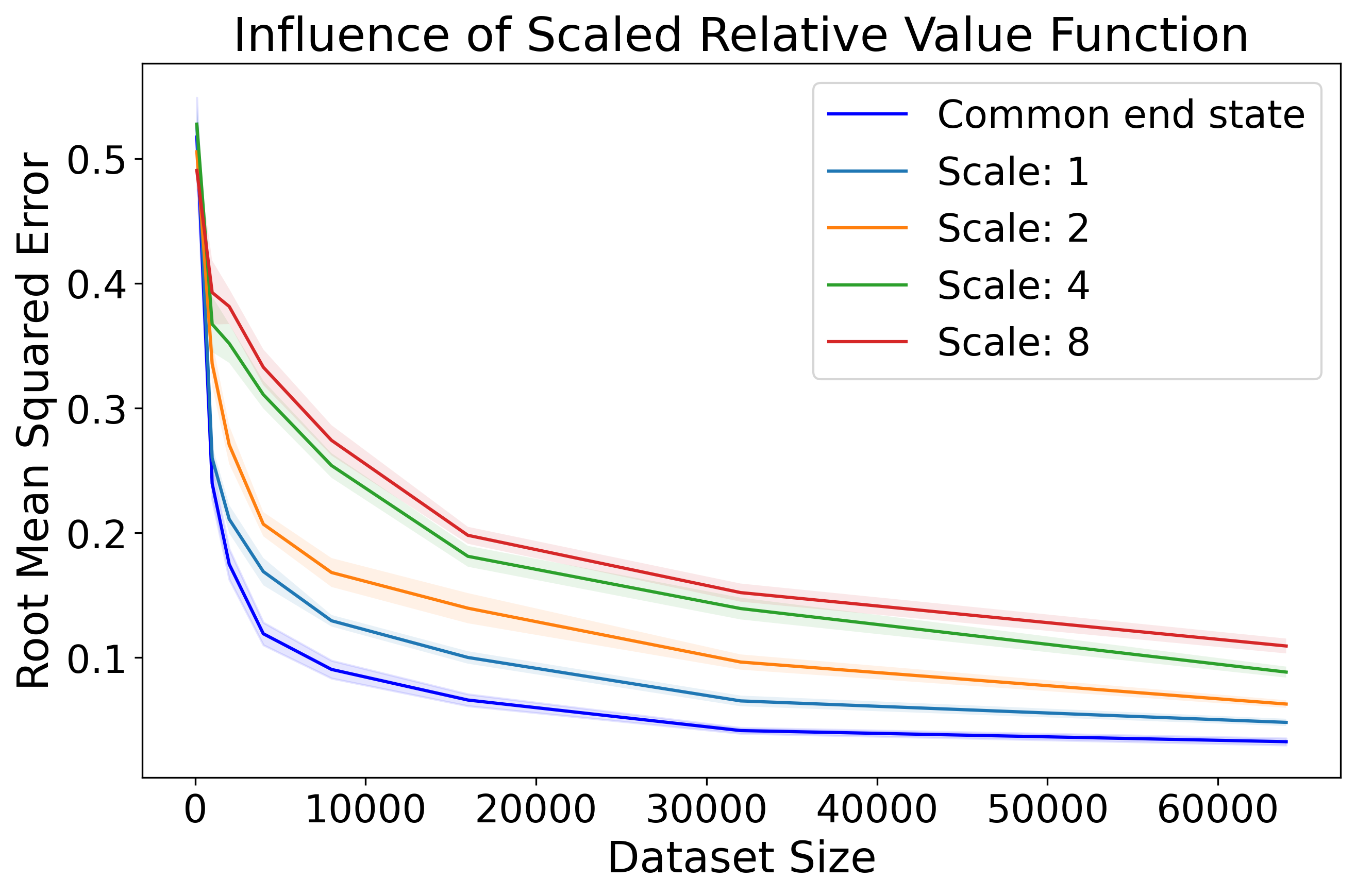}
    
    \begin{minipage}{0.49\textwidth}
        \centering
        (a)
    \end{minipage}
    \hfill
    \begin{minipage}{0.49\textwidth}
        \centering
        (b)
    \end{minipage}

    \caption{\textbf{(a)} Reward function estimation error across dataset sizes. As the dataset size grows, the common reward function is recovered across the two datasets, generated by Human $1$ and $2$, respectively. \textbf{(b)} Reward function estimation error as we vary the scale of the relative value function. Multiplying the relative value function by a larger value leads to slower learning. The benefits of comparing trajectories that terminate in the same state becomes especially pronounced in this case.}
    \label{fig:rmse_two_plots}
\end{figure}

\section{Robustness of Bootstrapped Return}

In this section, we discuss benefits to robustness of learning via the bootstrapped return rather than the partial return and cumulative advantage models.  These benefits are summarized in Table \ref{tab:model-robustness}.  The following sections will elaborate on and justify claims made in this table.

\begin{table}
\centering
\begin{tabular}{|c|c|c|c|c|}
\hline
\multicolumn{2}{|c|}{\multirow{2}{4em}{alignment}} & \multicolumn{3}{|c|}{learning model} \\
\cline{3-5}
\multicolumn{2}{|c|}{} & \cellcolor{red!50} partial return & \cellcolor{red!50} cumulative advantage & \cellcolor{green!50} bootstrapped return \\
\hline
\multirow{3}{*}{choice model} & partial return & \cellcolor{red!50} yes & \cellcolor{red!50} no & \cellcolor{green!50} yes \\
\cline{2-5}
& cumulative advantage & \cellcolor{red!50} yes/no & \cellcolor{red!50} yes/no & \cellcolor{green!50} yes/no \\
\cline{2-5}
& bootstrapped return & \cellcolor{red!50} no & \cellcolor{red!50} no & \cellcolor{green!50} yes \\
\hline
\end{tabular}
\caption{When learning a model from infinite choice data, will reward align human preferences?  Effective learning from choices based on cumulative advantage requires transition probabilities.  To accommodate that, for the middle row, we assume the agent knows correct ones.  Answers indicate whether alignment is attained based on choices made by a human with correct/incorrect transition probabilities.  When the answer does not vary with correctness/incorrectness, a single answer is provided.  The shading signifies robustness of learning via bootstrapped return relative to alternatives.}
\label{tab:model-robustness}
\end{table}

\subsection{Choice Based on Partial Return}
\label{sec:robustness-data-partial-return}

We will now assume that choices are based on partial return. This corresponds to the first row in Table \ref{tab:model-robustness}. We first note that if the agent knows that choices are made based on partial return, then under suitable conditions, the agent can clearly recover $\succeq$.

We now ask: if the agent assumes that choices are based upon  bootstrapped return, will it still recover a reward function that expresses $\succeq$?

Fortunately, the answer is yes. To see this, note that the bootstrapped return $\sum_{t=0}^{T-1} r(s_t, a_{t+1}) + V(s_T)$ simplifies to the partial return $\sum_{t=0}^{T-1} r(s_t, a_{t+1})$ if we set $V(s) = 0$ for all $s \in \mathcal{S}$. Thus, the bootstrapped return choice model is a strict generalization of the partial return model. Therefore, if the human makes choices based on partial return, a learning model based on bootstrapped return will recover a reward function that expresses $\succeq$.

What if the agent instead uses a learning model based on cumulative advantage? In that case, it is not guaranteed that the agent recovers a reward function that is aligned with the human's preferences. To see this, consider an MDP with two states $\mathrm{good}$ and $\mathrm{bad}$, and two actions $\mathrm{stay}$ and $\mathrm{move}$, as illustrated in Figure \ref{fig:GoodBadMDP}. The action Stay always keeps the agent in the same state and the action Move always takes the agent to the state that the agent is not in. Let the reward function be specified as in the table of Figure \ref{fig:GoodBadMDP}.

\begin{figure}[H]
\centering
\begin{subfigure}[c]{.35\textwidth}
  \centering
  \includegraphics[width=1.0\textwidth]{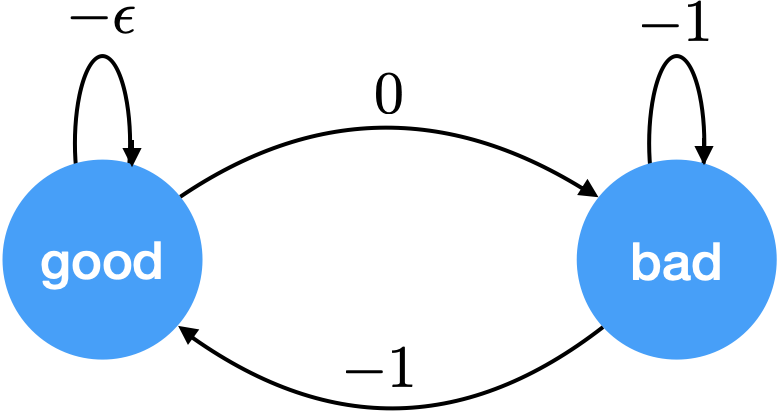}
\end{subfigure}
\hspace{0.5in}
\begin{subfigure}[c]{.35\textwidth}
\centering
\begin{tabular}{|c|c|c|}
\hline
$s$ & $a$ & $r(s,a)$ \\
\hline
$\mathrm{good}$ & $\mathrm{stay}$ & $-\epsilon$ \\
$\mathrm{good}$ & $\mathrm{move}$ & $0$ \\
$\mathrm{bad}$ & $\mathrm{stay}$ & $-1$ \\
$\mathrm{bad}$ & $\mathrm{move}$ & $-1$ \\
\hline
\end{tabular}
\end{subfigure}
\caption{An MDP for which learning via a cumulative advantage model from choices based on partial return produces a misaligned reward function.}
\label{fig:GoodBadMDP}
\end{figure}

We first note that the agent is misspecified in the following sense: the only way to estimate the choice probabilities perfectly would be to have $\Delta_* = r$. But $r$ is not an admissable advantage function, since in each state at least one action needs to have zero advantage. Therefore, we will consider an agent that is estimating $\Delta_*$ by searching amongst admissable advantage functions. 

Let us assume that the agent is trying to maximize likelihood. Further assume that the human is only showed trajectories of length $1$ sampled uniformly and that choices are sampled based on the logit choice model. Let $\hat{\Delta}$ be the estimated advantage. Since $r(\mathrm{good},\mathrm{stay}) < r(\mathrm{good},\mathrm{move})$, and the agent is maximizing likelihood it must be that $\hat{\Delta}(\mathrm{good},\mathrm{stay}) < \hat{\Delta}(\mathrm{good},\mathrm{move})$. Therefore, an agent that then treats $\hat{\Delta}$ as the advantage will then take action Move in state Good. This is clearly suboptimal since it takes the agent to state Bad where it will receive a large negative reward regardless of action. Thus, if choices are based on partial return, but the agent assumes choices as based on cumulative advantage, the agent may act according preferences unaligned with the human's.

\subsection{Choice Based on Cumulative Advantage}

We now turn attention to choices based on cumulative advantage.  This corresponds to the second row of Table \ref{tab:model-robustness}.  We separately treat cases where the human computes advantages using correct and incorrect transition probabilities.  In all cases, we assume that agent knows correct transition probabilities.

\subsubsection{Preliminaries}

We first introduce some notation and basic mathematical results that will facilitate our discussion.  For any state $s$, action $a$, policy $\pi$, transition matrix $P$, reward function $r$, and discount factor $\gamma \in (0,1)$, value functions are defined by
$$V_{\pi,P,r,\gamma}(s) = \E_{\pi,P}\left[\sum_{t=0}^\infty \gamma^t r(S_t,A_t) \Big| s\right]
\qquad \text{and} \qquad Q_{\pi,P,r,\gamma}(s,a) = r(s,a) + \gamma \sum_{s' \in \states} P_{ass'} V_{\pi,P,r,\gamma}(s'),$$
and the corresponding advantage function is
$\Delta_{\pi,P,r,\gamma}(s) = Q_{\pi,P,r,\gamma}(s,a) - V_{\pi,P,r,\gamma}(s)$.
As the discount factor approaches one, there is a well-defined limit, which we can then maximize over policies:
$$\Delta_{\pi,P,r}(s) = \lim_{\gamma \uparrow 1} \Delta_{\pi,P,r,\gamma}(s) \qquad \text{and} \qquad \Delta_{*,P,r}(s) = \max_\pi \Delta_{\pi,P,r}(s).$$
In Section \ref{se:introduction}, we denoted the optimal advantage function by $\Delta_*$.  The additional subscripts we introduce for this section make explicit the dependence on $P$ and $r$.

We define a sense in which a reward function $r$ expresses preferences between policies for a fixed transition matrix $P$.
\begin{definition}
\label{def:express-infinite-P}
{\bf (expressing $\succeq$ for $P$)}
A function $r:\states\times\actions\rightarrow \Re$ is said to express $\succeq$ on $P$ if, for all $\pi$ and $\pi'$ such that $(\pi,P) \in \mathcal{P}$ and $(\pi',P) \in \mathcal{P}$,
$$\ell \succeq \ell' \qquad \text{if and only if} \qquad \lim_{\gamma \uparrow 1} (V_{\pi, r, \gamma}(s) - V_{\pi', r, \gamma}(s)) \geq 0,$$
for all $s \in \states$, where $\ell,\ell' \in \mathcal{L}_\infty$ are lotteries induced by $(s,\pi,P)$ and $(s,\pi',P)$, respectively.
\end{definition}
If a learning algorithm recovers a reward function $\tilde{r}$ that expresses $\succeq$ on $P$, where $P$ is the correct transition matrix, then its use in comparing policies will be consistent with $\succeq$.  Note that the requirement that $(\pi,P)$ and $(\pi', P)$ are in $\mathcal{P}$ ensures that the preference between policies is consistent across all states $s$.

As \citet{knox2024learning} argue, the advantage function can sometimes substitute for the reward function without altering an agent's goals.  The following result, which is similar to results established by \citet{knox2024learning}, supports this notion.
\begin{restatable}[]{theorem}{advantage}
\label{th:advantage}
{\bf (advantages as rewards)} For all $r$, $\pi$, $\overline{\pi}$, and $\gamma \in (0,1)$,
$$V_{\pi, P, r,\gamma}(s) - V_{\pi, P, \tilde{r}, \gamma}(s) =  V_{\overline{\pi}, P, r, \gamma}(s),$$
where $\tilde{r} = \Delta_{\overline{\pi}, P, r, \gamma}$.
\end{restatable}
The $V_{\overline{\pi}, P, r, \gamma}(s)$ on the right hand side does not change with $\pi$.  Hence, the difference $V_{\pi, P, r,\gamma}(s) - V_{\pi, P, \tilde{r}, \gamma}(s)$ is constant across $\pi$.  In other words, discounted values assessed by $r$ and $\tilde{r}$ are consistent.  In particular, for each starting state, the ordering of policies by value is the same..  The following corollary follows almost immediately.
\begin{restatable}[]{corollary}{advantageCo}
\label{co:advantage}
{\bf (advantages express $\succeq$ on $P$)}  If $r$ expresses $\succeq$ then, for all $\pi$ and $P$, $\Delta_{\pi, P, r}$ expresses $\succeq$ on $P$.
\end{restatable}
A similar result holds for any reward function $\tilde{r}$ that induces the same advantage function as $r$.
\begin{restatable}[]{corollary}{advantageMatching}
\label{co:advantage-matching}
{\bf (imputed rewards express $\succeq$ on $P$)}  
For all $P$, $\overline{\pi}$, and $r$ that expresses $\succeq$, any $\tilde{r}$ for which $\Delta_{\overline{\pi}, P, r} = \Delta_{\overline{\pi}, P, \tilde{r}}$ expresses $\succeq$ on $P$. 
\end{restatable}

\subsubsection{Cumulative Advantage Based on Correct Transition Probabilities}

Let $P$ denote correct transition probabilities.  Recall that, with the cumulative advantage model, choice probabilities are determined by scores of the form $\sum_{t=0}^{T-1} \Delta_{*, P, r}(s_t,a_{t+1})$.  In this section, we will make a weaker assumption that the human uses scores of the form $\sum_{t=0}^{T-1} \Delta_{\overline{\pi}, P, r}(s_t,a_{t+1})$ for some policy that is not necessarily optimal.  We will establish that if the human knows $(\overline{\pi}, P)$ then an agent that knows $(\overline{\pi}, P)$ and learns via any of the three models we consider will recover a reward function $\tilde{r}$ that expresses $\succeq$ on $P$.  The results of this section extend those of \citet{knox2024models} and \citet{knox2024learning}.

First consider learning via the cumulative advantage model itself.  
If the agent can identify choice probabilities for every pair of trajectories, which is the case given infinite data of the right sort, then the agent can identify $\Delta_{\overline{\pi},P,r}$.  And since the agent knows $\overline{\pi}$ and $P$, the agent can impute a reward function $\tilde{r}$ such that $\Delta_{\overline{\pi},P,\tilde{r}} = \Delta_{\overline{\pi},P,r}$.  It follows from Corollary \ref{co:advantage-matching} that $\tilde{r}$ expresses $\succeq$ on $P$.

We next consider learning via the partial and bootstrapped return models.  Choices determined by the cumulative advantage $\sum_{t=0}^{T-1} \Delta_{\pi, P, r}(s_t,a_{t+1})$ are identical to ones based on the bootstrapped return $\sum_{t=0}^{T-1} \tilde{r}(s_t,a_{t+1}) + \tilde{V}(s_T)$ with with reward function $\tilde{r} = \Delta_{\pi, P, r}$ and value function $\tilde{V}=\mathbf{0}$.  Thes choices are also identical to ones based on partial return with reward function $\tilde{r} = \Delta_{\pi, P, r}$.  Consequently, learning via a bootstrapped return model will recover $(\tilde{r}, \tilde{V})$ and learning via a partial return model will recover $\tilde{r}$.  In either case, we obtain a reward function $\tilde{r} = \Delta_{\pi, P, r}$ that, by Corollary \ref{co:advantage}, expresses $\succeq$ on $P$.

\subsubsection{Cumulative Advantage Based on Incorrect Transition Probabilities}

We now consider choices made based on a transition matrix $\tilde{P} \neq P$, which expresses human beliefs about the environment.  In this case, choice probabilities are determined by scores of the form $\sum_{t=0}^{T-1} \Delta_{*, \tilde{P}, r}(s_t,a_{t+1})$.

When learning via a cumulative advantage model, the agent recovers $\Delta_{*, \tilde{P}, r}$ since these are determined by the choice probabilities. From this, the agent infers a reward function $\tilde{r}$ that results in the same advantage.  Importantly, however, when doing so, the agent uses $P$ rather than $\tilde{P}$. Thus, the recovered reward function $\tilde{r}$ satisfies $\Delta_{*, \tilde{P}, r} = \Delta_{*, P, \tilde{r}}$.  Note that the agent computes advantages $\Delta_{*, P, \tilde{r}}$ based on the true transition matrix $P$, which we have assumed the agent knows.  To optimize for preferences expressed by $\tilde{r}$, the agent would choose a policy that is greedy with respect to $\Delta_{*, P, \tilde{r}}$.  Since this policy is also greedy with respect to $\Delta_{*, \tilde{P}, r}$, it is optimal according to the human's beliefs about the environment.  But since these beliefs are erroneous, the agent may not pursue the true optimal policy.  In this way, learning via the cumulative advantage model produces a reward function that conflates goals and beliefs and is therefore fails to express $\succeq$ on $P$.

Now consider learning via the partial return.  In this case, with infinite data of the right sort, the agent recovers the reward function $\tilde{r} = \Delta_{*, \tilde{P}, r}$.  By Corollary \ref{co:advantage}, $r$ expresses $\succeq$ on $\tilde{P}$ and $\tilde{r}$ expresses $\succeq$ on $P$.  In other words, choosing a policy to optimize for $\tilde{r}$ and $P$ will induce desirable lotteries if states transition according to $\tilde{P}$.  However, such a policies may be undesirable if states transition according to $P$.

To make this point more concrete, let us revisit the example of Figure \ref{fig:1dGrid}.  Let $\tilde{P}$ represent the case where the grid wraps around, though according to the true transition matrix $P$ it does not.  Recall that $r$ assigns
large and positive when reaching the treasure and is otherwise small and negative.  The learned reward function $\tilde{r} = \Delta_{*, \tilde{P}, r}$ assigns high reward to moving left from the leftmost cell, though according to $r$ and $P$, that results in a self-transition and negative reward.  An agent that optimizes for $\tilde{r}$ and $P$ will move left from the leftmost cell, so $\tilde{r}$ fails to express $\succeq$ on $P$.  The same argument applies to learning via the bootstrapped return model.

\subsection{Choice Based on Bootstrapped Return}

We now consider the case when choices are based on bootstrapped return.  This corresponds to the third row of Table \ref{tab:model-robustness}.  We first note that if the agent knows that choices are made based on bootsrapped return, then under suitable conditions, the agent can clearly recover $\succeq$, as established in previous sections.

Let us consider the case when the agent assumes that choices are made by partial return. We first note that the agent is then misspecified in the sense that there is no way for the agent to perfectly estimate choice probabilities unless $V = 0$. A corollary is that what reward function that is recovered depends on the distribution over trajectory pairs shown to the human. We will leverage this fact. Consider some arbitrary MDP and consider two state-action pairs $(s,a)$ and $(s,a')$ with a shared state $s$. Assume that
\begin{align*}
    r(s,a) =& 0 \qquad r(s,a') = 0.5.
\end{align*}
Further, consider two states $s'$ and $s''$ that will serve as terminal states in the trajectories that we will construct. Assume that 
\begin{align*}
    V(s') =& -1 \qquad V(s'') = 1.
\end{align*}
We consider a dataset that grows indefinitely such that all trajectory pairs are shown an infinite number of times. We will show that it is possible to upweight certain trajectory pairs such that the agent infers a reward function estimate $\tilde{r}$ such that $\tilde{r}(s,a) > \tilde{r}(s,a')$, which would not be aligned with $r$. To see this, consider two trajectories $h = (s,a,s'')$ and $h' = (s,a',s')$. Since $h$ has a higher bootstrapped return than $h'$, $h$ will tend to be preferred over $h'$. Therefore, if $h$ and $h'$ show up as a pair in the dataset at a sufficiently high rate, the agent will infer that $\tilde{r}(s,a) > \tilde{r}(s,a')$. But this reward function clearly does not give the same ordering over trajectory lotteries as $r$. 

Further, even if the human had the correct value function, the above argument goes through. The reason is that as long as $V$ takes on values that are sufficiently large, you can ensure that the inferred $\tilde{r}$ orders state-action pairs in any way. Loosely speaking, if you want a state action pair to be ranked higher, you show that state action pair together with a terminal state with high value, more often. 

Finally, we ask: what if the agent assumes that choices are based on cumulative advantage? In this case, the agent does not necessarily recover a reward function that is aligned with the human's preferences. For instance, if the value estimates are zero everywhere, then choice data is effectively generated by partial return. Therefore, the same argument as in Section \ref{sec:robustness-data-partial-return} applies.


\section{Closing Remarks}

We have proposed and motivated a model of choice between partial trajectories based on bootstrapped return. Unlike with partial return, choices based on bootstrapped return are influenced by the human's beliefs, and thus provide a more realistic model of human choice. When choices are made based on bootstrapped return, we show that it is still possible for the agent to infer the human's goals, regardless of the human's beliefs. In this sense, the agent can disentangle the human's goals (as expressed by the reward function) from their beliefs about future observations. This is unlike when choices are based on cumulative advantage, which relies on the human having correct beliefs about the environment.

Disentangling goals from beliefs is also often an aim of inverse reinforcement learning (IRL)  \citep{baker2009action,evans2016learning,reddy2018you,shah2019feasibility}. In IRL, the agent aims to infer goals from human actions, rather than from human choices between trajectories. Indeed, our work relies heavily on the fact that trajectories include not only actions but outcomes.  In particular, we assume that the human can to some extent judge from a trajectory whether their goals are being met so that their choice communicates some information about that.   In IRL, on the other hand, the only information from the human is in the actions they execute, which does not necessarily inform the agent about whether goals are being met.  For this reason, disentangling goals from beliefs may require weaker assumptions when learning from choices between partial trajectories.   

More similar to our work, \citet{gong2020you} shows how the agent may infer the wrong reward function from trajectory ratings unless it accounts for potentially inaccurate beliefs of the human. Somewhat similar to \citep{knox2024models}, they model the human as basing their rating on how the trajectory deviates from what the human thinks is optimal behavior. It seems reasonable that bootstrapped return could also be used as a basis for modeling how humans rate (rather than compare) trajectories. It would be interesting to compare such a model to that of \citep{gong2020you}.

Our work adds to the growing chorus of work on RLHF emphasizing that current methods use unrealistic models of human behavior (see, e.g., \citep{tien2023causal, knox2024models} and, for a survey, \citep{casper2023open}). Unlike previous works, our work focuses on how human beliefs about future rewards may influence choice between trajectories. One issue pertaining to human beliefs, but not addressed in our work, is that of partial observability which was recently discussed by \citet{lang2024your}: when making comparisons between trajectories, the human often does not have full access to the agent's trajectory. This may incentivize the agent to pursue trajectories such that the observation sequence revealed to the human looks good, while the underlying state sequence is actually not. Combining their work with a choice model based on bootstrapped return may be an interesting avenue for future work.  

There is a clear connection between our work and \citet{savage1972foundations}’s axiomatization of subjective expected utility maximization. In his framework, goals are represented by a utility function rather than a reward function, and beliefs by a subjective probability distribution over states, rather than over future infinite trajectories. While the choice setting considered by Savage is very different from ours, it shares the basic idea of choices being influenced by both goals and beliefs. One interesting avenue for future work is to, in Savage’s spirit, derive our two axioms in terms of more basic axioms that reference neither reward nor probability.

We mention three possible extensions of the choice model we have proposed. First, rather than modeling the human as having fixed transition matrix in mind, they could be modeled as being uncertain, with a belief distribution over transition matrices. Then, when assessing desirability of a trajectory, they may account for how observing the trajectory would update this belief distribution.  Second, one strong assumption made in this work is that the human's reward function is defined in terms of an observed state. We believe that it should be possible to relax this assumption to accommodate a broader set of preferences while still retaining identifiability of $\succeq$. Third, our model could easily be extended to also allow for choices based on what  \citet{knox2024models} calls \textit{deterministic regret}. This can easily be done by adding to the bootstrapped return another quantity computed for the starting state. This added quantity results in a choice model that is strictly more general, as the bootstrapped return is recovered as a special case.

Finally, we note that we expect accounting for human beliefs to become increasingly important as tasks become increasingly complex.  In particular, we expect that bootstrapped return to play a more important role in complex tasks with delayed consequences that are difficult for humans to anticipate.

\section*{Acknowledgements}

This research was supported by a seed fund from the Stanford Institute for Human-Centered AI, contributed by SCBX and SCB10X, and a grant from the US Army Research Office.  We thank Derya Cansever, Micah Carroll, Stephane Hatgis-Kessell, Anmol Kagrecha, Saurabh Kumar, Ramesh Johari, and Tanwa Arpornthip for stimulating conversations and helpful feedback.

\bibliographystyle{plainnat}
\bibliography{references}

\begin{thebibliography}{23}
\providecommand{\natexlab}[1]{#1}
\providecommand{\url}[1]{\texttt{#1}}
\expandafter\ifx\csname urlstyle\endcsname\relax
  \providecommand{\doi}[1]{doi: #1}\else
  \providecommand{\doi}{doi: \begingroup \urlstyle{rm}\Url}\fi

\bibitem[Akrour et~al.(2012)Akrour, Schoenauer, and Sebag]{akrour2012april}
Riad Akrour, Marc Schoenauer, and Mich{\`e}le Sebag.
\newblock April: Active preference learning-based reinforcement learning.
\newblock In \emph{Machine Learning and Knowledge Discovery in Databases:
  European Conference, ECML PKDD 2012, Bristol, UK, September 24-28, 2012.
  Proceedings, Part II 23}, pages 116--131. Springer, 2012.

\bibitem[Baker et~al.(2009)Baker, Saxe, and Tenenbaum]{baker2009action}
Chris~L Baker, Rebecca Saxe, and Joshua~B Tenenbaum.
\newblock Action understanding as inverse planning.
\newblock \emph{Cognition}, 113\penalty0 (3):\penalty0 329--349, 2009.

\bibitem[Bastianello and Faro(2019)]{bastianello2019time}
Lorenzo Bastianello and Jos{\'e}~Heleno Faro.
\newblock Time discounting under uncertainty.
\newblock \emph{arXiv preprint arXiv:1911.00370}, 2019.

\bibitem[Brown and Niekum(2019)]{brown2019deep}
Daniel~S Brown and Scott Niekum.
\newblock Deep {Bayesian} reward learning from preferences.
\newblock \emph{arXiv preprint arXiv:1912.04472}, 2019.

\bibitem[Casper et~al.(2023)Casper, Davies, et~al.]{casper2023open}
Stephen Casper, Xander Davies, et~al.
\newblock Open problems and fundamental limitations of reinforcement learning
  from human feedback.
\newblock \emph{Transactions on Machine Learning Research}, 2023.

\bibitem[Christiano et~al.(2017)Christiano, Leike, Brown, Martic, Legg, and
  Amodei]{christiano2017deep}
Paul~F Christiano, Jan Leike, Tom Brown, Miljan Martic, Shane Legg, and Dario
  Amodei.
\newblock Deep reinforcement learning from human preferences.
\newblock \emph{Advances in neural information processing systems}, 30, 2017.

\bibitem[Evans et~al.(2016)Evans, Stuhlm{\"u}ller, and
  Goodman]{evans2016learning}
Owain Evans, Andreas Stuhlm{\"u}ller, and Noah Goodman.
\newblock Learning the preferences of ignorant, inconsistent agents.
\newblock In \emph{Proceedings of the AAAI Conference on Artificial
  Intelligence}, volume~30, 2016.

\bibitem[Gong and Zhang(2020)]{gong2020you}
Ze~Gong and Yu~Zhang.
\newblock What is it you really want of me? {Generalized} reward learning with
  biased beliefs about domain dynamics.
\newblock In \emph{Proceedings of the AAAI Conference on Artificial
  Intelligence}, volume~34, pages 2485--2492, 2020.

\bibitem[Hadfield-Menell et~al.(2016)Hadfield-Menell, Russell, Abbeel, and
  Dragan]{hadfield2016cooperative}
Dylan Hadfield-Menell, Stuart~J Russell, Pieter Abbeel, and Anca Dragan.
\newblock Cooperative inverse reinforcement learning.
\newblock \emph{Advances in neural information processing systems}, 29, 2016.

\bibitem[Ibarz et~al.(2018)Ibarz, Leike, Pohlen, Irving, Legg, and
  Amodei]{ibarz2018reward}
Borja Ibarz, Jan Leike, Tobias Pohlen, Geoffrey Irving, Shane Legg, and Dario
  Amodei.
\newblock Reward learning from human preferences and demonstrations in {Atari}.
\newblock \emph{Advances in neural information processing systems}, 31, 2018.

\bibitem[Knox et~al.(2024{\natexlab{a}})Knox, Hatgis-Kessell, Adalgeirsson,
  Booth, Dragan, Stone, and Niekum]{knox2024learning}
W~Bradley Knox, Stephane Hatgis-Kessell, Sigurdur~Orn Adalgeirsson, Serena
  Booth, Anca Dragan, Peter Stone, and Scott Niekum.
\newblock Learning optimal advantage from preferences and mistaking it for
  reward.
\newblock In \emph{Proceedings of the AAAI Conference on Artificial
  Intelligence}, volume~38, pages 10066--10073, 2024{\natexlab{a}}.

\bibitem[Knox et~al.(2024{\natexlab{b}})Knox, Hatgis-Kessell, Booth, Niekum,
  Stone, and Allievi]{knox2024models}
W~Bradley Knox, Stephane Hatgis-Kessell, Serena Booth, Scott Niekum, Peter
  Stone, and Alessandro Allievi.
\newblock Models of human preference for learning reward functions.
\newblock \emph{Transactions on Machine Learning Research}, 2024{\natexlab{b}}.

\bibitem[Koopmans(1960)]{koopmans1960stationary}
Tjalling~C Koopmans.
\newblock Stationary ordinal utility and impatience.
\newblock \emph{Econometrica: Journal of the Econometric Society}, pages
  287--309, 1960.

\bibitem[Koopmans(1972)]{koopmans1972representation}
Tjalling~C Koopmans.
\newblock Representation of preference orderings over time.
\newblock \emph{Decision and organization}, 57:\penalty0 100, 1972.

\bibitem[Lang et~al.(2024)Lang, Foote, Russell, Dragan, Jenner, and
  Emmons]{lang2024your}
Leon Lang, Davis Foote, Stuart Russell, Anca Dragan, Erik Jenner, and Scott
  Emmons.
\newblock When your {AI}s deceive you: Challenges of partial observability in
  reinforcement learning from human feedback.
\newblock \emph{arXiv preprint arXiv:2402.17747}, 2024.

\bibitem[Puterman(2014)]{puterman2014markov}
Martin~L Puterman.
\newblock \emph{Markov decision processes: discrete stochastic dynamic
  programming}.
\newblock John Wiley \& Sons, 2014.

\bibitem[Reddy et~al.(2018)Reddy, Dragan, and Levine]{reddy2018you}
Sid Reddy, Anca Dragan, and Sergey Levine.
\newblock Where do you think you're going? {Inferring} beliefs about dynamics
  from behavior.
\newblock \emph{Advances in Neural Information Processing Systems}, 31, 2018.

\bibitem[Sadigh et~al.(2017)Sadigh, Dragan, Sastry, and
  Seshia]{sadigh2017active}
Dorsa Sadigh, Anca Dragan, Shankar Sastry, and Sanjit Seshia.
\newblock \emph{Active preference-based learning of reward functions}.
\newblock UC Berkeley, 2017.

\bibitem[Savage(1972)]{savage1972foundations}
Leonard~J Savage.
\newblock \emph{The foundations of statistics}.
\newblock Courier Corporation, 1972.

\bibitem[Shah et~al.(2019)Shah, Gundotra, Abbeel, and
  Dragan]{shah2019feasibility}
Rohin Shah, Noah Gundotra, Pieter Abbeel, and Anca Dragan.
\newblock On the feasibility of learning, rather than assuming, human biases
  for reward inference.
\newblock In \emph{International Conference on Machine Learning}, pages
  5670--5679. PMLR, 2019.

\bibitem[Sutton and Barto(2020)]{sutton2020reinforcement}
Richard~S Sutton and Andrew~G Barto.
\newblock \emph{Reinforcement Learning: An Introduction}.
\newblock The MIT Press, second edition, 2020.

\bibitem[Tien et~al.(2023)Tien, He, Erickson, Dragan, and
  Brown]{tien2023causal}
Jeremy Tien, Jerry Zhi-Yang He, Zackory Erickson, Anca~D Dragan, and Daniel~S
  Brown.
\newblock Causal confusion and reward misidentification in preference-based
  reward learning.
\newblock In \emph{International Conference on Machine Learning}. PMLR, 2023.

\bibitem[Von~Neumann and Morgenstern(1947)]{von1947theory}
John Von~Neumann and Oskar Morgenstern.
\newblock Theory of games and economic behavior, 2nd rev.
\newblock 1947.

\end{thebibliography}

\newpage
\appendix

\section{Proofs}

\subsection{Average and Relative Reward}

Before proving results from the body of the paper, we provide a few helpful lemmas, together with their proofs when nontrivial and nonstandard.  Our first lemma follows from the fact that, by construction of $\mathcal{L}_\infty$, for each $\ell \in \mathcal{L}_\infty$ and $s \in\states$, $\lim_{T \rightarrow \infty} \frac{1}{T} \E_\ell\left[\sum_{t=0}^{T-1} \1(S_t=s)\right]$ exists.
\begin{lemma}
For all $r:\states\times\actions\rightarrow \Re$ and $\ell \in \mathcal{L}_\infty$, 
$$\overline{r}_{\ell,r} = \lim_{T \rightarrow \infty} \frac{1}{T} \E_\ell\left[\sum_{t=0}^{T-1} r(S_t,A_t)\right]$$
is well-defined and finite.
\end{lemma}
Going forward, we take $\overline{r}_{\ell,r}$ to be defined as in this lemma.
If a lottery $\ell$ is induced by $(\pi,P,h)$, we alternatively write $\overline{r}_{\pi,P,r,h} = \overline{r}_{\ell,r}$.  Recall that if $(\pi,P) \in \mathcal{P}$ then $\lim_{T\rightarrow \infty} \E[\frac{1}{T} \1(S_t=s')|s]$ does not depend on $s$.  This implies the following result.
\begin{lemma}
For all $(\pi,P) \in \mathcal{P}$, $\overline{r}_{\pi,P,r,h}$ does not depend on $h$.
\end{lemma}
Since $\overline{r}_{\pi,P,r,h}$ does not depend on $h$, we will suppress the dependence and simply write $\overline{r}_{\pi,P,r}$.

Let
$$\tilde{r}_{\pi,P,r}(s,a) = r(s,a) - \overline{r}_{\pi,P,r}.$$
We refer to $\overline{r}_{\pi,P,r}$ as the {\it average reward} and $\tilde{r}_{\pi,P,r}$ as the {\it relative reward}.

\subsection{Relative Value}

Our next lemma introduces a notion of relative value for a lottery induced by $(\pi,P,h)$.  Intuitively, this is the incremental value a lottery offers due to starting with the partial trajectory $h$ rather than others.  That the relative value is well-defined and finite follows from the fact that expected state frequencies do note depend on the initial state.
\begin{lemma}
For all $(\pi,P) \in \mathcal{P}$, $h \in \mathcal{H}$, and $r:\states\times\actions \rightarrow \Re$, if $\ell$ is the lottery induced by $(\pi,P,h)$ then
$$\tilde{v}_{\ell,r} = \lim_{\gamma\uparrow 1} \E_{\pi,P}\left[\sum_{t=0}^\infty \gamma^t \tilde{r}_{\ell,r}(S_t,A_t)\Big| h\right]$$
is well-defined and finite. 
\end{lemma}
Going forward, we take $\tilde{v}_{\ell,r}$ to be defined as in this lemma.

The following lemma specializes the previous one to the case where $h=(s)$ and indicates that $\tilde{v}_{\ell,r}$ can be decomposed into a sum of relative rewards generated by $h$ and the value of a subsequent lottery.
\begin{lemma}
\label{le:RelativeValue}
For all $(\pi,P) \in \mathcal{P}$, $s \in \states$, and $r:\states\times\actions \rightarrow \Re$,
$$\tilde{V}_{\pi,P,r}(s) = \lim_{\gamma \uparrow 1} \E_{\pi,P}\left[\sum_{t=0}^\infty \gamma^t \tilde{r}_{\pi,P,r}(S_t,A_t) \Big| s\right]$$
is well-defined and finite. Further, for any $h = (s_0,a_1,\ldots,s_T) \in \mathcal{H}$, if $\ell$ is the lottery induced by $\pi$, $P$, and $h$ then $$\tilde{v}_{\ell,r} = \sum_{t=0}^{T-1} \tilde{r}_{\pi,P,r}(s_t,a_t) + \tilde{V}_{\pi,P,r}(s_T).$$
\end{lemma}
Going forward, we take $\tilde{V}_{\pi,P,r}(s)$ to be defined as in this lemma.
Note that the decomposition follows from simple algebra:
\begin{align*}
\tilde{v}_{\ell,r} 
=& \lim_{\gamma\uparrow 1} \E_{\pi,P}\left[\sum_{t=0}^\infty \gamma^t \tilde{r}_{\ell,r}(S_t,A_t)\Big| h\right] \\
=& \sum_{t=0}^{T-1} \tilde{r}_{\ell,r}(s_t,a_t) + \lim_{\gamma\uparrow 1} \E_{\pi,P}\left[\sum_{t=T}^\infty \gamma^t \tilde{r}_{\ell,r}(S_t,A_t)\Big| h\right] \\
=& \sum_{t=0}^{T-1} \tilde{r}_{\ell,r}(s_t,a_t) + \lim_{\gamma\uparrow 1} \gamma^T \E_{\pi,P}\left[\sum_{t=0}^\infty \gamma^t \tilde{r}_{\ell,r}(S_t,A_t)\Big| s_T\right] \\
=& \sum_{t=0}^{T-1} \tilde{r}_{\pi,P,r}(s_t,a_t) + \tilde{V}_{\pi,P,r}(s_T)
\end{align*}
where $\ell$ is the lottery induced by $\pi$, $P$ and $h$.

\subsection{Lottery Preferences via Relative Value}

Recall from Definition \ref{def:express-infinite} that $r$ is said to express $\succeq$ if, for all $\ell,\ell' \in \mathcal{L}_\infty$, $\ell \succeq \ell'$ if and only if $\lim_{\gamma \uparrow 1} (v_{\ell,r,\gamma} - v_{\ell',r,\gamma}) \geq 0$.  The following lemma establishes that the condition can be substituted by $\tilde{v}_{\ell,r} \geq \tilde{v}_{\ell',r}$, which does not involve taking a limit.

\begin{lemma}
\label{le:PreferencesViaRelativeValue}
If $r$ expresses $\succeq$ then, for all lotteries $\ell$ and $\ell'$ induced by some $(\pi,P) \in \mathcal{P}$ with $h = (s_0,a_0,\ldots,s_T)$ and $h' = (s'_0,a'_0,\ldots,s'_{T'})$, respectively, 
$$\ell \succeq \ell' \quad \text{if and only if} \quad \tilde{v}_{\ell,r} \geq \tilde{v}_{\ell',r}.$$
\end{lemma}

\begin{proof}
We have
\begin{align*}
\tilde{v}_{\ell,r} - \tilde{v}_{\ell',r}
=& \lim_{\gamma\uparrow 1} \left(\E_{\pi,P}\left[\sum_{t=0}^\infty \gamma^t \tilde{r}_{\ell,r}(S_t,A_t)\Big| h\right] - \E_{\pi,P}\left[\sum_{t=0}^\infty \gamma^t \tilde{r}_{\ell',r}(S_t,A_t)\Big| h\right]\right) \\
=& \lim_{\gamma\uparrow 1} \left(\E_{\pi,P}\left[\sum_{t=0}^\infty \gamma^t (r_{\ell,r}(S_t,A_t) - \overline{r}_{\ell,r}) \Big| h\right] - \E_{\pi,P}\left[\sum_{t=0}^\infty \gamma^t (r_{\ell',r}(S_t,A_t) - \overline{r}_{\ell',r})\Big| h\right]\right) \\
\stackrel{(a)}{=}& \lim_{\gamma\uparrow 1} \left(\E_{\pi,P}\left[\sum_{t=0}^\infty \gamma^t r_{\ell,r}(S_t,A_t) \Big| h\right] - \E_{\pi,P}\left[\sum_{t=0}^\infty \gamma^t r_{\ell',r}(S_t,A_t) \Big| h\right]\right) \\
=& \lim_{\gamma \uparrow 1} (v_{\ell,r,\gamma} - v_{\ell',r,\gamma}).
\end{align*}
Because $(\pi, P) \in \mathcal{P}$, we have that $\overline{r}_{\ell,r} = \overline{r}_{\ell',r}$, which means that $(a)$ follows. Since $\tilde{v}_{\ell,r} - \tilde{v}_{\ell',r} = \lim_{\gamma \uparrow 1} (v_{\ell,r,\gamma} - v_{\ell',r,\gamma})$, we have that $\tilde{v}_{\ell,r} \geq \tilde{v}_{\ell',r}$ if and only if $\lim_{\gamma \uparrow 1} (v_{\ell,r,\gamma} - v_{\ell',r,\gamma}) \geq 0$.
\end{proof}

\subsection{Relationship Between Expressions of $\succeq_\partial$}

Recall from Definition \ref{def:express-infinite} that $(r,V)$ is said to express $\succeq_\partial$ if, for all partial histories $h=(s_0,a_0,\ldots,s_T)$ and $h'=(s'_0,a'_0,\ldots,s'_{T'})$, $h \succeq_\partial h'$ if and only if
$$\sum_{t=0}^T r(s_t,a_t) + V(s_T) \geq \sum_{t=0}^{T'} r(s'_t,a'_t) + V(s'_{T'}).$$
In other words, preferences between partial trajectories are determined by bootstrapped returns.  The following lemma establishes that the reward function that expresses these preferences is unique up to positive affine transformations.
\begin{lemma}
\label{le:UniquenessOfExpression}
If $(r,V)$ and $(r',V')$ each express $\succeq_\partial$ then there exists $\alpha \in \Re_{++}$ and $\beta \in \Re$ such that, for all $s \in \states$ and $a \in \actions$, $r'(s,a) = \alpha r(s,a) + \beta$.
\end{lemma}
\begin{proof}
Fix a distinguished state $\overline{s}$ and let $\overline{\mathcal{H}}_T$ be the set of partial trajectories that end at $\overline{s}$.  For all $(s,a,\overline{s}), (s',a',\overline{s}) \in \overline{\mathcal{H}}_1$, if $\succeq_\partial$ is indifferent then $r(s,a) = r(s',a')$.  If $\succeq_\partial$ indicates indifference for every pair in $\overline{\mathcal{H}}_1$ then $r$ and $r'$ are each a constant function, and therefore, they differ by a constant.  The result is then satisfied with $\alpha=1$ and $\beta$ equal to the amount by which $r'$ exceeds $r$.

Now suppose that $\succeq_\partial$ does not express indifference for every pair in $\overline{\mathcal{H}}_1$ and, thus, $r$ and $r'$ are not constant functions.  Let $(s^*,a^*)\in \argmax_{s,a} r(s,a)$.  Note that $(s^*,a^*)\in \argmax_{s,a} r'(s,a)$ because otherwise there would be there would be some $(s,a,\overline{s}) \in \overline{\mathcal{H}}_1$ preferred according to $r'(s,a) \succeq r'(s^*,a^*)$ but not $r(s,a) \succeq r(s^*,a^*)$.  Similarly, let $(s_*, a_*) \in \argmin_{s,a} r(s,a)$, and note that $(s_*, a_*) \in \argmin_{s,a} r'(s,a)$.

Let
$$\hat{r}(s,a) = \frac{r(s,a) - r(s_*,a_*)}{r(s^*,a^*) - r(s_*,a_*)} \qquad \hat{r}'(s,a) = \frac{r'(s,a) - r'(s_*,a_*)}{r'(s^*,a^*) - r'(s_*,a_*)}.$$
Note that $r$ and $r'$ are positive affine transformations of $\hat{r}$ and $\hat{r}'$.
Thus, for each $T$, $\hat{r}$ and $\hat{r}'$ express $\succeq_\partial$ over $\overline{\mathcal{H}}_T$. Further, each function has maximum one and minimum zero.  In particular, $\hat{r}(s^*,a^*) = \hat{r}'(s^*,a^*) = 1$ and $\hat{r}(s_*,a_*) = \hat{r}'(s_*,a_*) = 0$.

We will next establish that, for all $s$ and $a$, $\hat{r}(s,a) = \hat{r}'(s,a)$.  If $\hat{r}(s,a) = 1$ then $\succeq_\partial$ is indifferent between $(s,a,\overline{s})$ and $(s_*,a_*,\overline{s})$, and therefore $\hat{r}'(s,a) = 1$.  Similarly, if $\hat{r}(s,a) = 0$ then $\hat{r}'(s,a) = 0$.  We now handle rewards between zero and one.  Assume for contradiction that that one of the functions assigns a larger value to $(s,a)$ than the other.  Say, $1 > r(s,a) > r'(s,a) > 0$.  Let $\theta = (r(s,a) + r'(s,a))/2$.  
For each duration $T$, consider two elements of $\mathcal{H}_T$: $h_T = (s, a, s, a, \ldots, s, a, \overline{s})$ and $H_T = (S_0,A_0,\ldots, S_{T-1}, A_{T-1}, S_T=\overline{s})$ with each  state-action pair sampled independently according to
$$(S_t,A_t) = \left\{\begin{array}{ll}
(s^*,a^*) \qquad &\text{with probability } \theta \\
(s_*, a_*) \qquad &\text{otherwise.}
\end{array}\right.$$
By the law of large numbers, with probability one, $\lim_{T\rightarrow \infty} \frac{1}{T} \sum_{t=0}^{T-1} r(S_t,A_t) = \theta \in (\hat{r}'(s,a), \hat{r}(s,a))$.
It follows that, with probability one, the infinite trajectory $(S_0,A_0,S_1,A_1,\ldots)$ is such that there exists a trajectory-dependent $T$ such that, for all $T' \geq T$,
$$\sum_{t=0}^{T'-1} \hat{r}(s,a) > \sum_{t=0}^{T'-1} \hat{r}(S_t,A_t) = \sum_{t=0}^{T'-1} \hat{r}'(S_t,A_t) > \sum_{t=0}^{T'-1} \hat{r}'(s,a).$$
Hence, $\hat{r}$ prefers $h_T$ over $H_T$ whereas $\hat{r}'$ prefers $H_t$ over $h_T$.  Since both $\hat{r}$ and $\hat{r}'$ express $\succeq_\partial$ over $\mathcal{H}_T$, this yields a contradiction.  It follows that $\hat{r} = \hat{r}'$.

We know that $\hat{r}$ and $\hat{r}'$ express the same preferences over $\mathcal{H}_T$ as do $r$ and $r'$.  Since $\hat{r} = \hat{r}'$ and $r$ and $r'$ can be obtained via a positive affine transformation, the result follows.
\end{proof}

\subsection{The Alignment Theorem}

Building on the preceding lemmas, we prove the Alignment Theorem.
\alignment*

\begin{proof}
Let $r_*$ express $\succeq$.  Axiom \ref{ax:reward-representation} ensures existence of such a reward function.

We first establish Statement 1. By Axiom \ref{ax:alignment} there exists $(\pi, P) \in \mathcal{P}$ such that for all partial trajectories $h,h'$,
\begin{align*}
    h \succeq_\partial h' \qquad &\Leftrightarrow \qquad \underbrace{\mathbb{P}_{\pi, P}(\cdot | h)}_{\ell} \succeq \underbrace{\mathbb{P}_{\pi, P}(\cdot | h')}_{\ell'} \\
                                 &\stackrel{(a)}{\Leftrightarrow} \qquad \tilde{v}_{\ell, r_*} \geq \tilde{v}_{\ell', r_*} \\
                                 &\stackrel{(b)}{\Leftrightarrow} \qquad \sum_{t=0}^{T-1} \tilde{r}_{\pi,P,r_*}(s_t,a_t) + \tilde{V}_{\pi,P,r_*}(s_T) \geq \sum_{t=0}^{T'-1} \tilde{r}_{\pi,P,r_*}(s'_t,a'_t) - \tilde{V}_{\pi,P,r_*}(s'_{T'})
\end{align*}
The steps (a) and (b) follows by Lemma \ref{le:PreferencesViaRelativeValue}, and Lemma \ref{le:RelativeValue}, respectively. It follows that $(\tilde{r}_{\pi,P,r_*}, \tilde{V}_{\pi,P,r_*})$ expresses $\succeq_\partial$.  Note that we use $\ell$ and $\ell'$ as shorthand for the lotteries induced by $(\pi,P) \in \mathcal{P}$ with $h$ and $h'$, respectively. 

We next establish Statement 2.  From our argument used to prove Statement 1, we know that for some $(\pi,P) \in \mathcal{P}$, $(\tilde{r}_{\pi,P,r_*}, \tilde{V}_{\pi,P,r_*})$ expresses $\succeq_\partial$.  By Lemma \ref{le:UniquenessOfExpression}, for any $(r,V)$ that expresses $\succeq_\partial$, $r$ is a positive affine transformation of $\tilde{r}_{\pi,P,r_*}$.  Recall that, for all $s$ and $a$, $\tilde{r}_{\pi,P,r_*}(s,a) = r_*(s,a) - \overline{r}_{\pi,P,r_*}$.  Hence, $r$ is a positive affine transformation of $r_*$.  It follows that $r$ expresses $\succeq$.
\end{proof}

\subsection{Relationships between Advantages and Rewards}

\advantage*
\begin{proof}
For any $\gamma \in (0,1)$ and $s \in \states$,
\begin{align*}
V_{\pi, P, \tilde{r}, \gamma}(s) 
=& \E_{\pi, P}\left[\sum_{t=0}^\infty \gamma^t \tilde{r}(S_t,A_t) \Big| s\right] \\
=& \E_{\pi, P}\left[\sum_{t=0}^\infty \gamma^t \Delta_{\overline{\pi}, P, r, \gamma}(S_t,A_t) \Big| s\right] \\
=& \E_{\pi, P}\left[\sum_{t=0}^\infty \gamma^t (Q_{\overline{\pi}, P, r, \gamma}(S_t,A_t) - V_{\overline{\pi}, P, r, \gamma}(S_t)) \Big| s\right] \\
=& \E_{\pi, P}\left[\sum_{t=0}^\infty \gamma^t (r(S_t,A_t) + \gamma V_{\overline{\pi}, P, r, \gamma}(S_{t+1}) - V_{\overline{\pi}, P, r, \gamma}(S_t)) \Big| s \right] \\
=& V_{\pi, P, r, \gamma}(s) - V_{\overline{\pi}, P, r, \gamma}(s).
\end{align*}
\end{proof}

\advantageMatching*

\begin{proof}
Let $\pi$ and $\pi'$ be policies such that $(\pi,P), (\pi',P) \in \mathcal{P}$ and, for any state $s$, $\ell \succeq \ell'$ for lotteries $\ell,\ell' \in \mathcal{L}_\infty$ induced by $(s,\pi,P)$ and $(s,\pi',P)$.

It follows from Theorem \ref{th:advantage} that, for any $\overline{\pi}$,
$$V_{\pi, P, r,\gamma}(s) - V_{\overline{\pi}, P, r, \gamma}(s) = V_{\pi, P, \Delta_{\overline{\pi},P,r}, \gamma}(s) = V_{\pi, P, \Delta_{\overline{\pi},P,\tilde{r}}, \gamma}(s) = V_{\pi, P, \tilde{r},\gamma}(s) - V_{\overline{\pi}, P, \tilde{r}, \gamma}(s).$$
Rearranging terms of the equation between the first and last expressions gives us
$$V_{\pi, P, r,\gamma}(s) - V_{\pi, P, \tilde{r},\gamma}(s) =  V_{\overline{\pi}, P, r, \gamma}(s) - V_{\overline{\pi}, P, \tilde{r}, \gamma}(s).$$
Because the right-hand-side does not depend on $\pi$, it follows that for any other policy $\pi'$,
$$V_{\pi, P, r,\gamma}(s) - V_{\pi, P, \tilde{r},\gamma}(s) =  V_{\pi', P, r,\gamma}(s) - V_{\pi', P, \tilde{r},\gamma}(s),$$
and therefore,
\begin{align*}
V_{\pi, P, \tilde{r},\gamma}(s) - V_{\pi', P, \tilde{r},\gamma}(s) 
=& V_{\pi, P, r,\gamma}(s) - V_{\pi', P, r,\gamma}(s) \\
\lim_{\gamma \uparrow 1} (V_{\pi,P,\tilde{r},\gamma}(s) - V_{\pi',P,\tilde{r},\gamma}(s)) =& \lim_{\gamma \uparrow 1} (V_{\pi,P,r,\gamma}(s) - V_{\pi',P,r,\gamma}(s)) \geq 0.
\end{align*}
\end{proof}

\section{Measure-Theoretic Considerations}
In the main text, we left out details on exactly what sigma-algebra the lotteries are defined on.   Further, in the main text we used expectations with respect to those lotteries, but never discussed whether these expectations are well-defined. In this section, we fill in the details, and show that the expectations we use in the paper are well-defined.

\subsection{Measurable Space}
To formally define probability distributions over infinite trajectories, we need to define a measurable space that we can define our measures on. To do so, we first equip $\mathcal{H}_\infty$ with the product topology denoted by $\mathcal{T}(\mathcal{H}_\infty)$. We then define our sigma-algebra to be the Borel-sigma with respect to $\mathcal{T}(\mathcal{H}_\infty)$. We denote this sigma-algebra by $\mathcal{B}(\mathcal{H}_\infty)$. All probability distributions are defined with respect to the measurable space $\left(\mathcal{H}_\infty, \mathcal{B}(\mathcal{H}_\infty)\right)$. Note that for all partial trajectories $h_t \in \mathcal{H}$ the prefix set $\{ h_\infty' \in \mathcal{H}_\infty | h_t' = h_t \}$ is measurable.

\subsection{Expectations are well-defined}
In the paper, we compute expectations $\mathbb{E}_\ell[g(H_\infty)]$ with respect to some lottery $\ell$ and function $g: \mathcal{H}_\infty \to \Re$. For instance, we often compute this expectation with $g = \sum_{t=0}^\infty \gamma^t r(s_t,a_t)$, for some $\gamma \in (0,1)$ and $r: \mathcal{S} \times \mathcal{A} \to \Re$. We will prove that for this choice of $g$ this expectation is well-defined. 

We will define an infinite sequence of functions $\{ g_T \}$ as follows. For all $T$, let $g_T: \mathcal{H}_\infty \to \Re$ be defined as
$g_T(h_\infty) = \sum_{t=0}^{T-1} \gamma^t r(s_t,a_t)$ for all $h_\infty \in \mathcal{H}_\infty$. Then,
\begin{enumerate}
    \item $g_T$ converges to $g$ pointwise as $T \to \infty$.
    \item $g_T$ is absolutely bounded by $\sum_{t=0}^\infty \gamma^t \max_{s,a} |r(s,a)|$.
    \item $g_T$ is measurable.
\end{enumerate}

Therefore, by the Dominated Convergence Theorem, $g$ is integrable, and the expectation is thus well-defined. The same argument can be made for when $g$ computes the average reward or the state visitation frequency.

\end{document}